\documentclass[nohyperref]{article}

\usepackage{microtype}
\usepackage{graphicx}
\usepackage{subfigure}
\usepackage{booktabs} % for professional tables

\usepackage{hyperref}

\usepackage[accepted]{icml2023}

\usepackage{amsmath}
\usepackage{amssymb}
\usepackage{mathtools}
\usepackage{amsthm}

\usepackage[capitalize,noabbrev]{cleveref}
\crefname{assumption}{Assumption}{Assumptions}

\usepackage{enumitem}

\usepackage{xspace}
\usepackage{thmtools}
\RequirePackage{dsfont}
\RequirePackage{delimset}

\newcommand{\R}{\mathbb{R}} %reals
\newcommand{\N}{\mathbb{N}} % naturals
\newcommand{\E}{\mathbb{E}} % expectation
\newcommand{\I}{\mathbb{I}} %indicator
\newcommand{\poly}{\mathrm{poly}}

\DeclareMathOperator*{\argmax}{arg\,max}

\declaretheoremstyle[
	    spaceabove=\topsep, 
	    spacebelow=\topsep, 
	    headfont=\normalfont\bfseries,
	    bodyfont=\normalfont\itshape,
	    notefont=\normalfont\bfseries,
	    notebraces={(}{)},
	    postheadspace=0.33em, 
	    headpunct={.},
    ]{theorem}
\declaretheorem[style=theorem]{theorem}

\declaretheoremstyle[
	    spaceabove=\topsep, 
	    spacebelow=\topsep, 
	    headfont=\normalfont\bfseries,
	    bodyfont=\normalfont,
	    notefont=\normalfont\bfseries,
	    notebraces={(}{)},
	    postheadspace=0.33em, 
	    headpunct={.},
    ]{definition}

\declaretheoremstyle[
        spaceabove=\topsep, 
        spacebelow=\topsep, 
        headfont=\normalfont\bfseries,
        bodyfont=\normalfont,
        notefont=\normalfont\bfseries,
        notebraces={}{},
        postheadspace=0.33em, 
        qed=$\blacksquare$, 
        headpunct={.},
    ]{proofstyle}
\declaretheorem[style=proofstyle,numbered=no,name=Proof]{proof}

\declaretheorem[style=theorem,sibling=theorem,name=Lemma]{lemma}
\declaretheorem[style=theorem,sibling=theorem,name=Corollary]{corollary}
\declaretheorem[style=theorem,sibling=theorem,name=Claim]{claim}

\declaretheorem[style=theorem,numbered=no,name=Theorem]{theorem*}
\declaretheorem[style=theorem,numbered=no,name=Lemma]{lemma*}
\declaretheorem[style=theorem,numbered=no,name=Corollary]{corollary*}
\declaretheorem[style=theorem,numbered=no,name=Proposition]{proposition*}
\declaretheorem[style=theorem,numbered=no,name=Claim]{claim*}
\declaretheorem[style=theorem,numbered=no,name=Fact]{fact*}
\declaretheorem[style=theorem,numbered=no,name=Observation]{observation*}
\declaretheorem[style=theorem,numbered=no,name=Conjecture]{conjecture*}

\declaretheorem[style=definition,sibling=theorem,name=Definition]{definition}

\declaretheorem[style=definition,numbered=no,name=Definition]{definition*}
\declaretheorem[style=definition,numbered=no,name=Remark]{remark*}
\declaretheorem[style=definition,numbered=no,name=Example]{example*}
\declaretheorem[style=definition,numbered=no,name=Question]{question*}

\declaretheorem[style=theorem,sibling=theorem,name=Assumption]{assumption}
\newcommand{\D}{\mathcal{D}} % distribution
\newcommand{\C}{\mathcal{C}} % context space
\newcommand{\F}{\mathcal{F}} % function class 
\newcommand{\Fp}{\mathcal{P}} % DYNAMICS class 
\newcommand{\M}{\mathcal{M}} % mdp
\newcommand{\Hist}{\mathbb{H}} % history
\newcommand{\Mhat}{\widehat{\mathcal{M}}}

\newcommand{\PiOptReg}[1][t]{\widetilde{\pi}^{c_{#1}}_{#1}}
\newcommand{\Regrv}{\mathcal{R}}
\newcommand{\OLSR}{\mathcal{O}_{\mathrm{sq}}}
\newcommand{\OLLR}{\mathcal{O}_{\mathrm{\log}}}
\newcommand{\RegSQ}{\mathcal{R}^\mathrm{sq}}
\newcommand{\RegLL}{\mathcal{R}^{\log}}

\newcommand{\Esq}{
\mathcal{E}_T(\ell_{\mathrm{sq}})}
\newcommand{\Ehel}{\mathcal{E}_T(D^2_H)}

\usepackage[normalem]{ulem}

\newif\ificml
\icmltrue % adds icml specific commands in paper

\icmltitlerunning{Efficient Rate-Optimal Regret for Adversarial CMDPs}

\begin{document}

\twocolumn[
\icmltitle{Efficient Rate Optimal Regret for Adversarial Contextual MDPs Using Online Function Approximation}



\icmlsetsymbol{equal}{*}

\begin{icmlauthorlist}
\icmlauthor{Orin Levy}{sch}
\icmlauthor{Alon Cohen}{yyy,comp}
\icmlauthor{Asaf Cassel}{sch}
\icmlauthor{Yishay Mansour}{sch,comp}
\end{icmlauthorlist}

\icmlaffiliation{yyy}{School of Electrical Engineering, Tel Aviv University, Tel Aviv, Israel}
\icmlaffiliation{comp}{Google Research, Tel Aviv, Israel}
\icmlaffiliation{sch}{Blavatnik School of Computer Science, Tel Aviv University, Tel Aviv, Israel}

\icmlcorrespondingauthor{Orin Levy}{orinlevy@mail.tau.ac.il}
\icmlcorrespondingauthor{Alon Cohen}{alonco@tauex.tau.ac.il}
\icmlcorrespondingauthor{Asaf Cassel}{acassel@mail.tau.ac.il}
\icmlcorrespondingauthor{Yishay Mansour}{mansour.yishay@gmail.com}

\icmlkeywords{Reinforcement Learning Theory, Regret Minimization, Online Learning, Adversarial Contextual MDPs, Online Function Approximation}
\vskip 0.3in
]
\printAffiliationsAndNotice{}  % leave blank if no need to mention equal contribution
% =================== ICML Document Configuration =================== %

\begin{abstract}
We present the OMG-CMDP! algorithm for regret minimization in adversarial Contextual MDPs. The algorithm operates under the minimal assumptions of realizable function class and access to online least squares and log loss regression oracles. Our algorithm is efficient (assuming efficient online regression oracles), simple and robust to approximation errors. It enjoys an
$\widetilde{O}(H^{2.5} \sqrt{ T|S||A| ( \mathcal{R}_{TH}(\mathcal{O})
+ H \log(\delta^{-1}) )})$ regret guarantee, 
with $T$ being the number of episodes, $S$ the state space, $A$ the action space, $H$ the horizon and $\mathcal{R}_{TH}(\mathcal{O}) = \mathcal{R}_{TH}(\mathcal{O}_{\mathrm{sq}}^\mathcal{F}) + \mathcal{R}_{TH}(\mathcal{O}_{\mathrm{log}}^\mathcal{P})$ 
is the sum of the square and log-loss regression oracles' regret, used to approximate the context-dependent rewards and dynamics, respectively. To the best of our knowledge, our algorithm is the first efficient rate optimal regret minimization algorithm for adversarial CMDPs that operates under the minimal standard assumption of online function approximation. 
\end{abstract}

\section{Introduction}

Reinforcement Learning (RL) is a framework for sequential decision making in unknown environments. 
A Markov Decision Process (MDP) is the mathematical model behind RL environments.
In the classical episodic RL setting, in each episode an agent repeatedly interacts with the MDP for $H$ steps, observing the current state and then choosing an action. Subsequently, the agent receives a reward and the process transitions to the next state. The agent's goal is to choose actions as to maximize the cumulative return.
This model characterizes many applications including online advertising, robotics, games and healthcare, and has been extensively studied over the past three decades (see e.g.,~\citealp{Sutton2018,MannorMT-RLbook}).
%
% MDPs suitably characterize many real life applications including: online advertising, robotics, games and even healthcare.  In the last three decades, MDPs were extensively studied in both online and offline learning scenarios (see e.g.,~\citealp{Sutton2018,MannorMT-RLbook}).

However, the classical MDP model cannot efficiently capture the influence of additional side information on the environment.
% , which characterizes many modern applications. 
Consider, for example, a medical trial that tests the effect of a new medication for a disease. The reaction of a patient to the treatment, which is modeled by the environment, is deeply influenced by the patient's medical history, which includes age, weight and background diseases she or he might have.
We refer to such external information, which is unaffected by the agent's decisions, i.e., treatment choice, as the patient's \emph{context}. 
% One possible approach is to encode it as part of the state. 
A standard MDP can encode the context as part of the state.
However, as no two patients are identical, this can significantly increase the size of the state space, and hence the
% Hence, this approach has the disadvantage of greatly increasing the state space, and by that the 
complexity of learning and even the complexity of representing a single policy.
Instead, Contextual MDPs (CMDPs), keep the state space small and treat the context as side information, revealed to the agent at the start of each episode. 
% CMDPs define a mapping from each context to an MDP, and the optimal policy for each context is an optimal policy of the related MDP.
Additionally, there is an unknown mapping from each context to an MDP, thus an optimal policy maps each context to an optimal policy of the related MDP.

As previously alluded to, the context space is often prohibitively large, prompting the use of function approximation. This has previously been studied in the context of Contextual Multi Armed Bandits (CMAB; \citealp{foster2020beyond,simchi2021bypassing}), and more recently in CMDPs~\cite{foster2021statistical,levy2022optimism}. Concretely, realizable function approximation implies that the agent is provided with a function class of mappings from contexts to MDPs, and the goal is to obtain learning guarantees in terms of the class' complexity rather than the size of the context space. Realizability further implies that the true mapping resides in the class.
%
%The function class is typically accessed via a regression oracle that makes predictions based on the observed interaction. 
% A key distinction between oracles is whether they provide \emph{online} or \emph{offline} guarantees. An offline guarantee, also known as an Empirical Risk Minimizer (ERM) implies that each prediction minimizes a loss function computed over the observations. An online or regret guarantee, implies that the cummulative prediction loss is not significantly larger than that of the best prediction in hindsight. 
% \ABC{online vs offline oracles}

In recent years, CMDPs have gained much interest. There are two major lines of works. The distinctive feature between these two lines is whether the context is \emph{stochastic} or \emph{adversarially chosen}.
\citet{hallak2015contextual} were the first to study CMDPs assuming an adversarial context.
\citet{modi2020no} considered adversarial contexts and a Generalized Linear Model as a function class and gave a $\smash{\sqrt{T}}$ regret guarantee.
\citet{foster2021statistical} consider general function classes as part of their Estimation to Decision (E2D) framework. They assume an access to an online estimation oracle that finds the best fit over the function class given past contexts and observations. 
(A detailed discussion and comparison with this work 
appears in the sequel.)
For stochastic CMDPs,~\citet{levy2022optimism} presented sublinear regret under the minimum reachability assumption, using access to \emph{offline} least squares regression oracles.
%Their algorithms regret was bounded $\widetilde{O}(\sqrt{T}/p_{min})$ where $p_{min}>0$ is a minimum reachability parameter of the CMDP, which can be arbitrarily small. 
They also showed an $\Omega(\sqrt{TH|S||A|\log(|\mathcal{G}|/|S|)/\log(|A|)})$ regret lower bound
%for the general setting of offline function approximation 
where $\abs{\mathcal{G}}$ is the size of the function class used to approximate the rewards.
Later,~\citet{levy2022counterfactul} showed regret of $\widetilde{O}(\sqrt{T})$ assuming an access to offline regression oracles without the reachability assumption. 
Clearly, adversarial CMDPs generalize stochastic CMDPs and require the use of online function approximation.
% More precisely, it gets as an input a context-dependent trajectory and the which policy generated it, and returns a distribution over context-dependent policies. 
% They assume an access to that oracle. Their work established that sub-linear regret can be achieved for adversarial CMDPs. However, it is unclear how to implement the online estimation oracle efficiently. In addition, their approach requires solving an optimization problem and use the optimal solution of it to derive the played policy. In practice, optimization algorithms return an approximation of the objective's optimum. It is unclear how the use of an approximation instead of the exact optimum would affect the regret guarantees of their method. 

\paragraph{Contributions.}
We study regret in adversarial CMDPs under a natural online function approximation setting.
We present the OMG-CMDP! algorithm, and prove that its regret is, with probability at least  $1-\delta$, bounded as 
$$
    %\widetilde{O}\left(
    H^{2.5} \sqrt{ T|S||A| \left( \Regrv_{TH}(\OLSR^\F) + \Regrv_{TH}(\OLLR^\Fp) + H \log\frac{1}{\delta}\right)}
    %\right)
$$
(up to poly-logarithmic factors of $T,|S|,|A|,H$) 
%with probability at least $1-\delta$,
where $S$ is the state space, $A$ the action space, $H$ the horizon and $\Fp$ and $\F$ are function classes used to approximate the context dependent rewards and dynamics, respectively. 
We assume access to online log loss and least squares regression oracles for the dynamics and rewards approximation, and $\Regrv_{TH}(\OLLR^\Fp)$, $\Regrv_{TH}(\OLSR^\F)$  denote their regret guarantees, respectively. 
Our algorithm performs only $2T$ oracle calls and its running time is in $\poly(|S|,|A|,H,T)$ assuming an efficient online regression oracles. 
The main advantage of our technique is its simplicity. We present an intuitive algorithm which operates under the standard online function approximation assumptions. 
In addition, at each round, our played policy can be approximated efficiently using standard convex optimization algorithms.
%which makes our algorithm implementable.

%Notice that, unlike states, contexts do not need to be explored. This will allow  \ABC{weaker assumption compared to large state RL }
%\Orin{Maybe this part is unnecessary ? :) We already added a discussion on the linear case}

\paragraph{Comparison with \citet{foster2021statistical}.} 
This work is most related to ours.
They obtained $\sqrt{T \;\Regrv_T(\mathcal{O}_{\text{est}})}$ regret by applying their Estimation to Decision (E2D) meta algorithm to adversarial CMDPs, where $\mathcal{O}_{\text{est}}$ is an online estimation oracle that operates over a class of CMDPs.
Their work, being very general, leaves their oracle implementation non-concrete and generic, consequently giving rise to relatively complex algorithmic machinery.
Specifically, at each round $t$ their algorithm outputs a distribution $p_t$ over context-dependent policies using an Inverse Gap Weighting (IGW) technique. To make this computationally tractable, they compute an approximate Policy Cover (PC) that serves as $p_t$'s support, and analyze the delicate interplay between the approximation error and the IGW technique.
In contrast, we use a standard online function approximation oracles for both the dynamics and rewards given related function classes, providing a clear model-based representation of the learned CMDP.
Our approach yields a natural and intuitive algorithm, only requiring an approximate solution of a convex optimization problem. 
Thus, it can be implemented efficiently.

\vspace{-1em}
\paragraph{Additional Related Literature.}\label{subsec:related-work}
% Regret in adversarial CMDPs was first studied by \citet{hallak2015contextual}. However, they assume a small context space, and their regret scales linearly in the context space cardinality.
% Later, \citet{modi2020no} present a regret bound of $\widetilde{O}(\sqrt{T})$ for CMDPs with a Generalized Linear Model (GLM) assumption.
% Our online function approximation framework is much more general then GLM.
%
Sample complexity bounds for Contextual Decision Processes (CDPs) have been studied under various assumptions. 
\citet{jiang2017contextual} present OLIVE, a sample efficient algorithm for learning Contextual Decision Processes (CDP) under the low Bellman rank assumption and later~\citet{sun2019model} show PAC bounds for model based learning of CDPs using the Witness Rank.

\citet{modi2018markov} present generalization bounds for learning \emph{smooth} CMDPs and finite contextual linear combinations of MDPs. 
We, in contrast, consider the regret of adversarial CMDPs.

% Recently, stochastic contextual MDPs have gained much attention. In this setting, an i.i.d contexts are sampled for each episode from an unknown distribution $\D$. 
\citet{levy2022learning} studied the sample complexity of learning stochastic CMDPs using a standard ERM oracle, and provided the first general and efficient reduction from Stochastic CMDPs to offline supervised learning.
%However, their sample complexity scales as $\smash{\epsilon^{-8}}$. 
%
% Later, \citet{levy2022optimism} presented sublinear regret for stochastic contextual MDPs under the minimum reachability assumption, assuming an access to an \emph{offline} least square regression oracles.
% %Their algorithms regret was bounded $\widetilde{O}(\sqrt{T}/p_{min})$ where $p_{min}>0$ is a minimum reachability parameter of the CMDP, which can be arbitrarily small. 
% They also showed an $\Omega(\sqrt{TH|S||A|\log(|\mathcal{G}|/|S|)/\log(|A|)})$ regret lower bound
% %for the general setting of offline function approximation 
% where $\abs{\mathcal{G}}$ is the size of the function class used to approximate the rewards.
% More recently,~\citet{levy2022counterfactul} removed the reachability assumption and presented the UC$^3$RL algorithm for regret minimization stochastic CMDPs using offline log loss and least square regression oracles. Their algorithm enjoys optimal rate regret of $\widetilde{O}(\sqrt{T})$.
%
%
%
%

%\newpage
CMDPs naturally extend the well-studied CMAB model.
% \noindent\textbf{Contextual Multi Armed Bandits}
CMABs augment the Multi-Arm Bandit (MAB) model with a context that determines the rewards \cite{MAB-book,Slivkins-book}. \citet{Langford2007,agarwal2014taming} use an optimization oracle, and give an optimal regret bound that depends on the size of the policy class they compete against. 
\citet{foster2021instance} present instance-dependent regret bounds for stochastic CMAB assuming access to a function class $\F$ for the rewards approximation.
%They also show instance dependent regret bounds for Block MDPs.
%
Regression based approaches appear in \citet{agarwal2012contextual,foster2020beyond,foster2018practical,foster2021logloss,simchi2021bypassing,xu2020upper} for both stochastic and adversarial CMABs.
%
%Most closely related to our work, 
\citet{foster2020beyond} assume access to an online least-squares regression oracle, and prove an optimal regret bound for adversarial CMABs using the IGW technique. 
%As mentioned, later~\citet{foster2021statistical} implemented the approach to CMDPs using online estimation oracle.
%We do not see how to explicitly extend IGW to CMDPs under the standard online function approximation setting.
%for both the rewards and dynamics. 
%of $\widetilde{O}\left(\sqrt{T |A| \Regrv_T(\OLSR^\F)}\right)$ regret, where $\F$ is a realizable function class used to approximate the reward and $\Regrv_T(\OLSR^\F)$ is the regret of the used online least square regression oracle with respect to $\F$. 
%They obtain the above bound using the Inverse Gap Weighting (IGW) technique. We do not see, how to explicitly extend the IGW technique for CMDPs under the standard online function approximation setting for both the rewards and dynamics. 
%However, there is indeed a tight connection between the IGW policy and our algorithm. We discuss this relation in~\cref{sec:algorithm}. 

\section{Preliminaries: Episodic Markov Decision Process (MDP)}
A (tabular) MDP is defined by a tuple $(S,A,P,r,s_0, H)$, where $S$ and $A$ are finite state and action spaces respectively;
$s_0\in S$ is the unique start state; $H \in \N$ is the horizon;
$P : S \times A \times S \to [0,1]$ is the dynamics which defines the probability of transitioning to state $s'$ given that we start at state $s$ and play action $a$; and $r(s,a)$ is the expected reward of performing action $a$ at state $s$.
An episode is a sequence of $H$ interactions where at step $h$, if the environment is at state $s_h$ and the agent plays action $a_h$ then the environment transitions to state $s_{h+1} \sim P(\cdot \mid s_h, a_h)$ and the agent receives reward $R(s_h, a_h) \in [0,1]$, sampled independently from a distribution $\D_{s_h,a_h}$ that satisfies $r(s_h, a_h) = \E_{\D_{s_h,a_h}} \brk[s]*{ R(s_h,a_h)}$. 
%\\
% We remark that w.l.o.g we assume the rewards and dynamics are time-invariant, i.e., independent of the time step $h$ inside an episode. To handle time-varying dynamics and rewards, all we need to do is consider function classes that are time variant, meaning, their input also includes the time step $h$. Hence, our results can be naturally extended to this setting. 

%
%
A \emph{stochastic and non-stationary policy} $ \pi = (\pi_h : S \to \Delta(A))_{h \in [H]}$ defines for each time step $h \in [H]$ a mapping from states to a distribution over actions. 
Given a policy $\pi$ and MDP 
    $
        M 
        = 
        (S,A,P,r,s_0, H)
    $, 
the
$h \in [H-1] $ stage value function of a state $s \in S_h$ is defined as 
\begin{align*}
    V^{\pi}_{M,h} (s)
    = 
    \mathbb{E}_{\pi, M} 
    \brk[s]*{\sum_{k=h}^{H-1} r(s_k, a_k)\;\Bigg|\; s_h = s }
    .
\end{align*} 
For brevity, when $h = 0$ we denote $V^{\pi}_{M,0}(s_0) := V^{\pi}_M(s_0)$, which is the expected cumulative reward under policy $\pi$ and its measure of performance. Let 
$
\pi^\star_M
\in
\argmax_{\pi}\{V^{\pi}_{M}(s_0)\}
$
denote an optimal policy for MDP $M$.
%It is well known that such a policy is optimal even among the class of stochastic and history dependent policies (see, e.g., \citealp{puterman2014markov,Sutton2018,MannorMT-RLbook}).

Furthermore, we consider the notation of~\emph{occupancy measures} (see, e.g.,~\citealp{zimin2013online}). 
Let $q_h(s,a \mid \pi, P)$ denote the probability of reaching state $s\in S$ and performing action $a\in A$ at time $h \in [H]$ of an episode generated using  policy $\pi$ and dynamics $P$. 
Let $\mu(P) \subseteq [0,1]^{H\times S \times A}$ denote the set of all occupancy measures defined by the dynamics $P$ and any stochastic policy $\pi$. 
$\mu(P)$ is defined as follows. Each $q \in \mu(P)$ satisfies the following three requirements altogether.
\begin{enumerate}[label=(\roman*),leftmargin=*]
    \item $q \in [0,1]^{H \times S \times A}$ and $\forall h \in [H]$, $q_h \in \Delta(S\times A)$;
    \item For all $s \in S$, $\sum_{a \in A}q_0(s,a) = \I[s=s_0]$; and
    \item For all $h \in [H-1]$ and $s \in S$, $\sum_{a \in A}q_{h+1}(s,a) = \sum_{(s',a') \in S \times A} P(s \mid s',a') q_h(s',a')$.
\end{enumerate}
 % \begin{align*}
 %    &\mu(P) =
 %    \Bigg\{  :\;
 %    \\
 %    &\sum_{a \in A}q_0(s,a) = \I[s=s_0],\;\forall s \in S;
 %    \\
 %    &\sum_{a \in A}q_{h+1}(s,a) = \sum_{s' \in S}\sum_{a' \in A}P(s|s',a')\cdot q_h(s',a'),
 %    \\
 %    &\forall h \in [H-1],s \in S
 %    \Bigg\}
 %    .
 % \end{align*}
%
%A vector $q \in \mu(P)$ is the occupancy measure related to the dynamics $P$ and the policy 
%
$\pi^q$ is the policy associated with occupancy measure $q \in \mu(P)$ and is defined as follows for all $h \in [H]$ and state-action pair $(s,a)\in S \times A$,
$$
    %\forall h \in [H],\; (s,a)\in S \times A:\;\;
    \pi^q_h(a|s)  = \frac{q_h(s,a)}{\sum_{a' \in A}q_h(s,a')}
    .
$$
If $\sum_{a' \in A}q_h(s,a') = 0$ then $\pi^q_h(a|s):=1/|A|$. 
In addition, note that $\mu(P)$ is a convex set. See, e.g.,~\citet{rosenberg2019online} for more details.

\section{Problem Setup: Adversarial Contextual MDP}
Following the notations of~\citet{levy2022optimism}, 
a \emph{CMDP} is defined by a tuple $(\mathcal{C},S, A, \mathcal{M})$ where $\mathcal{C}$ is the context space, $S$  the state space and $A$  the action space. The mapping $\mathcal{M}$  maps a context $c\in \mathcal{C}$ to an MDP
    $
        \mathcal{M}(c) 
        =
        (S, A, P^c_\star, r^c_\star,s_0, H)
    $, 
where $r^c_\star(s,a) = \E[R^c_\star(s,a) \mid c,s,a]$, $R^c_\star(s,a) \sim \D_{c,s,a}$.
We assume that $R^c_\star(s,a) \in [0,1]$.\\
For mathematical convenience, we assume the contexts space $\C$ is finite but potentially very large and we do not want to depend on its size (in both our regret bound and computation run-time). Our results are naturally extended to infinite contexts space.

We consider an \emph{adversarial} CMDP where in each episode the context can be chosen in a completely arbitrary manner, possibly by an adversary. 

A stochastic and non-stationary \emph{context-dependent policy} $\pi = \left( \pi^c \right)_{c \in \mathcal{C}}$ maps a context $c \in \mathcal{C}$ to a policy  $ \pi^c= (\pi^c_h : S \to \Delta(A))_{h \in [H]}$.

\paragraph{Interaction protocol.} The interaction between the agent and the environment is defined as follows.
    In each episode $t=1,2,...,T$:
    \begin{enumerate}[label=(\roman*),nosep]
        \item An adversary chooses a context $c_t \in \C$;
        %$c_t \sim \D$.
        \item The agent chooses a policy $\pi^{c_t}_t$ 
        %(based on $c_t$ and the observed history)
        ;
        \item The agent observes a trajectory generated by playing $\pi^{c_t}_t$  in $\M(c_t)$, denoted as
        $ 
            \sigma^t = (c_t, s^t_0, a^t_0, r^t_0, \ldots, s^t_{H-1}, a^t_{H-1},r^t_{H-1},s^t_H) 
        $.
    \end{enumerate}

Our goal is to minimize the regret, defined as
\begin{equation}\label{eq:regret-def}
    \Regrv_T 
    %(\text{Algorithm})
    := \sum_{t=1}^T 
    V^{\pi^{c_t}_\star}_{\M(c_t)}(s_0)
    -
    V^{\pi^{c_t}_t}_{\M(c_t)}(s_0),
\end{equation}
where
$
    \pi_\star 
$
is an optimal context-dependent policy.

\subsection{Assumptions}

In this setting, without further assumptions, the regret may scale linearly in $TH$. 
We overcome this limitation by imposing the following minimal online function approximation assumptions, which extend similar notions in the CMAB literature (see, e.g.,~\citealp{agarwal2012contextual,foster2018practical,foster2020beyond,foster2021logloss}) to CMDPs. We assume access to realizable function classes $\F$ and $\Fp$ that serve to approximate the context-dependent rewards and dynamics respectively.
Realizability means that the true rewards and dynamics belong to the appropriate function class, and access is via online regression oracles (more details later). 
% To apply these assumptions to our algorithm, we access $\F$ and $\Fp$ via online regression oracles. 

\noindent\textbf{Online regression oracle.}
We consider a standard online regression oracle with respect to a given loss function $\ell$. 
%and a function class $\F$ for the following online learning setting.
The oracle performs real-valued online regression where the examples are chosen from some subspace $\mathcal{Z} $,
with respect to a loss function $\ell$ and has a regret guarantee relative to the function class $\F$.
We consider the following online scenario.
For every round $t = 1, . . . , T$:
(1) An adversary (possibly adaptive) chooses input $z_t \in \mathcal{Z}$. 
(2) The oracle observes $z_t$ and returns a prediction $\hat{y}_t \in [0,1]$. 
(3) The adversary chooses an outcome $y_t \in [0,1]$.
% \\
%
We follow~\citet{foster2020beyond} and model the online oracle as a sequence of mappings $\mathcal{O}_\ell^t : \mathcal{Z} \times (\mathcal{Z}, \R)^{t-1} \to [0,1]$, where  
$  
        \hat{y}_t
        = 
        \mathcal{O}^t_\ell 
        (z_t ;  (z_1, y_1), ... , (z_{t-1}, y_{t-1}))
$. 
Each oracle implementation induces a
mapping $\hat{f}_t(z) = \mathcal{O}_\ell(z; (z_1,y_1),\ldots,(z_{t-1},y_{t-1}))$ which is the prediction for the input $z$ at round $t$. 
 %where $\hat{f}_t \in \F$. 
 (See Section 2.1 in~\citealp{foster2020beyond}).
The oracles' regret guarantee with respect to the function class $\F$ is as follows.
 \begin{align*}
    \sum_{t=1}^T \ell(\hat{y}_t,y_t)
    - \inf_{f \in \F}  \sum_{t=1}^T \ell(f(z_t), y_t)
    \leq 
    \Regrv_T(\mathcal{O}_\ell^\F)
    % \tag{Online Least Squares Regression (OLSR)}
    % \sum_{t=1}^T (\hat{y}_t - y_t)^2
    % - \inf_{f \in \F}  \sum_{t=1}^T (f(z_t) - y_t)^2
    % \leq 
    % \Regrv_T(\OLSR^\F)
    % ,
    % \\
    % \tag{Online Log Loss Regression (OLLR)}
    % \sum_{t=1}^T \log\left(\frac{1}{\hat{y}_t} \right)
    % - \inf_{P \in \Fp}  \sum_{t=1}^T \log\left(\frac{1}{P(z_t)} \right)
    % \leq 
    % \Regrv_T(\OLLR^\Fp)
    .
\end{align*}
In the following, we consider the online regression oracle with respect to the square loss i.e., $\ell_{sq}(\hat{y},y) = (\hat{y} - y)^2$ for the rewards approximation, and the log loss i.e., $\ell_{\log} (\hat{y},y) = \log ({y}/{\hat{y}})$ for the dynamics approximation.
 For finite function classes $\F$ and $\Fp$, it is known that the regret of these oracles is logarithmic in the function class size \citep{cesa2006prediction,foster2020beyond,foster2021statistical}. 
 Meaning, $\Regrv_T(\OLSR^\F) = {O}(\log (|\F|))$ and $\Regrv_T(\OLLR^\Fp) = {O}(\log (|\Fp|))$.\footnote{We remark that in the case that $\F$ or $\Fp$ are non-convex, some implementations of the oracles might return a function in the convex hull of $\F$ and $\Fp$, respectively. Such an implementation is, for instance, Vovk's aggregation algorithm~\citep{cesa2006prediction}. 
 %This would not affect any of our results.
 }
 The above optimization problems can always be solved by iterating over the function class. But, since we consider strongly convex loss functions, there are function classes where these optimization problems can be solved efficiently. An obvious example is the class of linear functions.  
 % We use the square-loss-based oracle to approximated the rewards, and the log-loss based oracle to approximate the dynamics.
 % We remark that for infinite function classes, the oracles' regret can be bounded by the sequential Rademacher complexity \citep{rakhlin2015sequential}, hence our results can be naturally extended to infinite function classes.

% \vspace{-1em}
\paragraph{Reward function approximation.}\label{par:reward-function-approx}
We assume that the learner has access to a class of reward functions
$\F\subseteq \C\times S \times A \to [0,1]$, each function $f \in \F$ maps context $c \in \C$, state $s \in S$ and an action $a \in A$ to a (approximate) reward $r \in [0,1]$.
We use $\F$ to approximate the context-dependent rewards function of any state $s\in S$ and action $a \in A$ using an  online least-squares regression (OLSR) oracle under the following realizability assumption.
\begin{assumption}\label{assumption:rewards-realizability}
    There exists a function $f_\star \in \F$ such that for all $t$ and $(s,a)\in S \times A$, $f_\star(c_t, s,a) = r^{c_t}_\star(s,a)$. 
\end{assumption}

\begin{assumption}[Square Loss Oracle Regret]\label{assumption:oracle-gurantee-rewards}
    The oracle $\OLSR^\F$ guarantees that for every sequence of trajectories $\left\{\sigma^t\right\}_{t=1}^T$, regret is bounded as
    \begin{align*}
        &\sum_{t=1}^T \sum_{h=0}^{H-1} (\hat{f}_t(c_t,s^t_h,a^t_h) - r^t_h)^2
        \\
        &\hspace{5mm}- \inf_{f \in \F}  \sum_{t=1}^T \sum_{h=0}^{H-1}  (f(c_t,s^t_h,a^t_h) - r^t_h)^2
        \leq 
        \Regrv_{TH}(\OLSR^\F).
    \end{align*}
\end{assumption}

\paragraph{Dynamics function approximation.}
For the unknown context-dependent dynamics case, our algorithm gets as input a function class $\Fp \subseteq  S \times A \times S \times \C \to [0,1]$, where every function $P \in \Fp$ satisfies 
$
    \sum_{s' \in S} P(s' \mid s,a,c)  = 1$ for all $c \in \C$ and $(s,a) \in S \times A
$.
We use $\Fp$ to approximate the context-dependent dynamics using an online log-loss regression (OLLR) oracle under the following realizability assumption. For any $P \in \Fp$ we denote $P^c(s' \mid s,a):= P(s' \mid s,a,c)$.
\begin{assumption}[Dynamics Realizability]\label{assumption:dynamics-realizability}
There exists a function $P \in \Fp$ such that for all $t$, and every $(s,a,s') \in S \times A \times S$, $P(s' \mid s,a,c_t) = P^{c_t}_\star(s' \mid s,a)$.
\end{assumption}

\begin{assumption}[Log Loss Oracle Regret]\label{assumption:oracle-KL-reget}
    Given a function class $\Fp$ of context-dependent transition probabilities function, the oracle $\OLLR^{\Fp}$ guarantees that for every sequence of trajectories $\left\{ (c_t; s^t_0,a^t_0,r^t_0, \ldots, s^t_H) \right\}_{t=1}^T$, regret is bounded as
    \begin{align*}
        &\sum_{t=1}^T \sum_{h=0}^{H-1} 
        \log\frac{1}{\widehat{P}^{c_t}_t(s^t_{h+1} \mid s^t_{h},a^t_h)}
        \\
        & 
        \hspace{3mm}- \inf_{P \in \Fp}  \sum_{t=1}^T \sum_{h=0}^{H-1}  \log\frac{1}{P^{c_t}(s^t_{h+1} \mid s^t_{h},a^t_h)}
        \leq 
        \Regrv_{TH}(\OLLR^{\Fp}).
    \end{align*}
\end{assumption}

\paragraph{Notations.} 
For an event $E$ we denote by $\I[E]$ the indicator function which returns $1$ if $E$ holds and $0$ otherwise. We denote expectation by $\E[\cdot]$. We also use the following abbreviations for the oracles regrets. We denote $\RegLL :=\Regrv_{TH}(\OLLR^{\Fp})$ and $\RegSQ :=\Regrv_{TH}(\OLSR^\F)$. 
Note that our oracles are called $T$ times, each time we feed them with $H$ examples. For that reason their regret bounds depend on $TH$. This in contrast to previous works in which the oracle gets only one example in each call.
When using the notation $\widetilde{O}(\cdot)$ we omit poly-logarithmic factors of $|S|,|A|,H,T$.
For a vector $x \in \R^d$ we denote by $\|x\|_1:= \sum_{i=1}^d |x_i|$ the $\ell_1$ norm of $x$.

\section{Algorithm and Main Result}\label{sec:algorithm}
We present the \textit{Occupancy Measures approximated reGularization algorithm for regret minimization in
 adversarial CMDP} (OMG-CMDP!; \cref{alg:OMG-CMDP}).
At each round $t = 1,2,\ldots, T$, we first approximate the rewards and dynamics, using the online oracles, based on the observed trajectories up to round $t-1$.
We denote by $\hat{f}_t$ the approximated rewards function, and by $\widehat{P}_t$ the approximated dynamics at round $t$. 
After observing the current context $c_t$, we solve the optimization problem in~\cref{eq:max-problem-unknown-dynamics} over the set of occupancy measures defined by $\widehat{P}^{c_t}_t$. For the optimal solution $\hat{q}^t$, we derive the appropriate policy $\pi^{c_t}_t$ and run it to generate a trajectory. We use the observed trajectory to update the oracles.

Note that, for the objective in \cref{eq:max-problem-unknown-dynamics} to be finite, we implicitly assume that for any round $t \ge 1$ there exists $q \in \mu(\widehat{P}_t^{c_t})$ such that $q > 0$ (entry-wise). This can always be achieved by mixing the oracle output with a uniform distribution where the weight of the uniform distribution can be arbitrarily small, thus has a negligible effect on the regret. Alternatively, we can exclude $(h, s, a)$ tuples that are unreachable under $\widehat{P}_t^{c_t}$ from the sum in the log barrier regularization, and define $\pi_{t,h}^{c_t}(a \mid s) = 1 / \abs{A}$ for these tuples. This would not change the analysis, but adds technical notations that reduce the clarity of the proofs, thus omitted.

\begin{algorithm}[!ht]
    \caption{Occupancy Measures approximated reGularization for adversarial CMDP (OMG-CMDP!)}
    \label{alg:OMG-CMDP}
    \begin{algorithmic}[1]
        \STATE
        { 
            \textbf{inputs:} 
            \begin{itemize}
                \item MDP parameters: $S$, $A$, $s_0$, $H$.
                \item Function classes $\F$ for rewards approximation and $\Fp$ for dynamics approximation. 
                \item Confidence parameter $\delta \in (0,1)$ and tuning parameter $\gamma$.
                \item OLSR oracle $\OLSR^\F$ and OLLR oracle $\OLLR^{\Fp}$.
            \end{itemize}
             
        } 
        \FOR{round $t = 1, \ldots, T$}
            \STATE{approximate rewards  
            $
                \hat{f}_t = \OLSR^\F
            $ and dynamics 
            $
                \widehat{P}_t = \OLLR^{\Fp}
            $.
            }
            \STATE{observe context $c_t \in \C$.}
            \STATE{solve
            \begin{equation}\label{eq:max-problem-unknown-dynamics}
                \begin{split}
                    \hat{q}^t = \argmax_{q \in \mu(\widehat{P}^{c_t}_t)}
                    &\sum_{\substack{(h,s,a) \in \\ [H] \times S \times A}} q_h(s,a)\cdot \hat{f}_t(c_t,s,a)
                    \\
                    &\hspace{5mm}+ \frac{1}{\gamma} \sum_{\substack{(h,s,a) \in \\ [H] \times S \times A}} \log(q_h(s,a)) 
                    .
                \end{split}
            \end{equation}
            }
            \STATE{derive policy as follows, for all $h \in [H]$ and $(s,a)\in S \times A$.
            \[
                %\forall h \in [H], (s,a)\in S \times A:\;\;
                \pi^{c_t}_{t,h}(a \mid s)  = \frac{\hat{q}^t_h(s,a)}{\sum_{a' \in A}\hat{q}^t_h(s,a')}
                .
            \]
            }
            \STATE{play $\pi^{c_t}_t$ and observe a trajectory $$\sigma^t =(c_t;s^t_0,a^t_0,r^t_0,\ldots,s^t_H).$$}
            \STATE{update $\OLSR^\F$ using $\{((c_t,s^t_h,a^t_h),r^t_h)\}_{h=0}^{H-1}$.}
            \STATE{update $\OLLR^{\Fp}$ using $\{(c_t,s^t_h,a^t_h,s^t_{h+1})\}_{h=0}^{H-1}$.}
            \ENDFOR
    \end{algorithmic}
\end{algorithm}
A key step in our algorithm is reflected in the optimization problem~(\cref{eq:max-problem-unknown-dynamics}).
We solve the maximization problem over the set of occupancy measures of the approximated dynamics $\widehat{P}^{c_t}_t$.
We remark that the maximization problem in~\cref{eq:max-problem-unknown-dynamics} is a strictly concave maximization problem over the set of occupancy measures of $\widehat{P}^{c_t}_t$. 
Hence it has a single global maximum that can be approximated efficiently.
%
%\paragraph{The novelty of our approach.}

% Note
We point out that usually in regret-minimization RL literature one optimizes over the empirical dynamics with additional bonuses to the rewards that promote exploration (optimism). 
The magnitude of the bonuses is derived directly from the size of the confidence intervals over the dynamics and rewards approximation. 
In our case, however, the contexts are adversarially chosen hence it is impossible to define such confidence intervals. 
Nevertheless, the use of log-barrier regularization provides the necessary trade-off between exploration and exploitation, in a manner resembling optimistic approaches, thus replacing the need for the aforementioned bonuses.
%
% However and perhaps surprisingly, we replace the use of confidence intervals by employing the log-barrier regularization over occupancy measures.
The application of log-barrier regularization differentiates us from previous works in contextual RL literature.

% Note that there is a close relation between the maximization problem in~\cref{eq:max-problem-unknown-dynamics} and IGW technique. If one runs our algorithm on an advarsarial CMAB instance, then the policy induced by the optimal solution of \cref{eq:max-problem-unknown-dynamics} is exactly the IGW policy. However, it is unclear if an explicit representation in an ``IGW formulation'' of such a policy for CMDPs is indeed exists. We bypass this obstacle by using the representation of a policy using occupancy measures, and derive the played policies from the occupancy measures computed using \cref{eq:max-problem-unknown-dynamics}.

Our main result is the regret guarantee of~\cref{alg:OMG-CMDP}, stated in the following theorem.
\begin{theorem}[Regret Bound]\label{thm:regret-main}
    Let $\delta \in (0,1)$. 
    For $\gamma = \sqrt{\frac{|S||A|T}{31H^3 \left( 2\RegSQ + \RegLL + 18H \log(2H/\delta) \right) }}$, 
    with probability at least $1-\delta$ it holds that $\Regrv_T (\text{OMG-CMDP!})$ is bounded by
    %bounded by
    \begin{align*}
        \widetilde{O}\left(H^{2.5} \sqrt{ T|S||A| 
        \left( 
        \RegSQ+ \RegLL + H \log \delta^{-1} \right)}\right)
        .
    \end{align*}
\end{theorem}
For finite function classes $\F$ and $\Fp$ the following corollary is immediately implied by~\cref{thm:regret-main}. 
\begin{corollary}
    Let $\F$ and $\Fp$ be finite function classes for the rewards and dynamics approximation, respectively. Let $\delta \in (0,1)$. There exists oracle implementations such that for an appropriate choice of $\gamma$, with probability at least $1-\delta$, $\Regrv_T (\text{OMG-CMDP!})$ is bounded by
    \begin{align*}
        \widetilde{O}\left(H^{2.5} \sqrt{ T|S||A| 
        \left( 
        \log|\F|+ \log|\Fp| + H \log \delta^{-1} \right)}\right)
        .
    \end{align*}
\end{corollary}

\section{Analysis}
Our regret analysis consists of two main conceptual steps.
The first step is to derive concentration bounds on our oracles' regret (\cref{sunsec:oracle-bounds-main}).
Next, note that the regret defined in~\cref{eq:regret-def} can be decomposed as follows,
\begin{alignat}{2}
        \Regrv_T
        \label{pi-star-both-models}
        &= 
        \sum_{t=1}^T 
        V^{\pi^{c_t}_\star}_{\M(c_t)}(s_0)
        -
        &&V^{\pi^{c_t}_\star}_{\Mhat_t(c_t)}(s_0)
        \\
        \label{pi-star-pi-t}
        &
        +
        \sum_{t=1}^T 
        V^{\pi^{c_t}_\star}_{\Mhat_t(c_t)}(s_0)
        -
        &&V^{\pi^{c_t}_t}_{\Mhat_t(c_t)}(s_0)
        \\
        \label{pi-t-both-models}
        &+
        \sum_{t=1}^T 
        V^{\pi^{c_t}_t}_{\Mhat_t(c_t)}(s_0)
        - 
        &&V^{\pi^{c_t}_t}_{\M(c_t)}(s_0)
        ,
\end{alignat}
where \cref{pi-star-both-models} is the cumulative difference between the value of the true optimal policy $\pi_\star$ on the true and approximated MDPs of the context $c_t$ at round $t$. \cref{pi-star-pi-t} is the cumulative difference between the value of $\pi_\star$ and $\pi_t$ on the approximated model at round $t$, for the context $c_t$. \cref{pi-t-both-models} is the cumulative difference between the value of the selected policy $\pi_t$ on the approximated and true MDPs at round $t$ for the context $c_t$.
%
%
% The second step is to bound the cumulative value-difference between the followings, using the expected oracles' regret.
%
% \begin{enumerate}[label=(\roman*),leftmargin=*,nosep]
%     \item The value of the true optimal policy $\pi_\star$ on the true and approximated MDPs of each context $c_t$ at round $t$, i.e.,
% $$
%     V^{\pi^{c_t}_\star}_{\M(c_t)}(s_0)-V^{\pi^{c_t}_\star}_{\Mhat_t(c_t)}(s_0)
%     .
% $$
%     \item The value difference between $\pi_\star$ and $\pi_t$ on the approximated model at round $t$, for the context $c_t$, i.e., 
% $$
%     V^{\pi^{c_t}_\star}_{\Mhat_t(c_t)}(s_0)-
%     V^{\pi^{c_t}_t}_{\Mhat_t(c_t)}(s_0)
%     .
% $$
%     \item The value of the selected policy $\pi_t$ on the approximated and true MDPs at round $t$ for the context $c_t$, i.e.,
% $$
%     V^{\pi^{c_t}_t}_{\Mhat(c_t)}(s_0)-V^{\pi^{c_t}_t}_{\M(c_t)}(s_0)
%     .
% $$
% \end{enumerate}
%
% Our regret is then bounded by the sum over each round $t \in [T]$ of the above three terms. 
In~\cref{subsec:val-diff-bounds-main} we upper bound the three value-difference sums in terms of the oracles' expected regret.  
By applying concentration bounds (\cref{lemma:lsro-regret,lemma:llrs-regret}), we bound the oracles' expected regret by the empirical one, thus bounding the regret of our algorithm with high probability.
%The proof can be found in~\cref{subsec:regret-bound-main}.

Throughout the analysis we consider the cumulative error caused by the dynamics approximation in terms of the Squared Hellinger distance between the true and approximated dynamics.
\begin{definition}[Squared Hellinger Distance]
\label{def:hellinger-main}
    For any two distributions $\mathbb{P}$, $\mathbb{Q}$ over a discrete support $X$ we define the Squared Hellinger Distance as
    $$
        D^2_H(\mathbb{P}, \mathbb{Q}) 
        := 
        \sum_{x \in X} \left(\sqrt{\mathbb{P}(x)} - \sqrt{\mathbb{Q}(x)}\right)^2
        .
    $$
\end{definition}
A useful property of the squared Hellinger distance is that for any two distributions $\mathbb{P}$ and $\mathbb{Q}$ it holds that 
\begin{equation}\label{eq:l-1-hellinger}
    \|\mathbb{P} - \mathbb{Q}\|^2_1 \leq 4 D^2_H(\mathbb{P}, \mathbb{Q}).
\end{equation}
We bound the cumulative value differences (\cref{pi-star-both-models,pi-star-pi-t,pi-t-both-models}) using the following quantities:

(1) The cumulative expected least-squares loss over each round $t$, i.e.,
    \begin{align*}
        \sum_{t=1}^T \mathop{\E}_{\pi^{c_t}_t, P^{c_t}_\star} \Bigg[ \sum_{h=0}^{H-1}
        \left(\hat{f}_t(c_t,s_h,a_h) - f_\star(c_t,s_h,a_h)\right)^2
        ~\bigg|s_0 \Bigg]
        .
    \end{align*}
    For abbreviation, we denote the above by $\Esq$.
    
(2) The cumulative expected squared Hellinger distance over each round $t$, i.e.,
    \begin{align*}
        \sum_{t=1}^T \mathop{\E}_{\pi^{c_t}_t, P^{c_t}_\star} \Bigg[
        \sum_{h=0}^{H-1} D^2_H(P^{c_t}_\star(\cdot|s_{h},a_{h}), \widehat{P}^{c_t}_t(\cdot|s_{h},a_{h}))
        ~\bigg|~ s_0\Bigg]
        .
    \end{align*}
    For abbreviation, we denote the above by $\Ehel$.

\subsection{Oracle Concentration Bounds}\label{sunsec:oracle-bounds-main}
The following lemma bounds the expected regret of the online least squares regression oracle by its realized regret.
\begin{lemma}[Concentration of OLSR regret]\label{lemma:lsro-regret}
    Under~\cref{assumption:rewards-realizability,assumption:oracle-gurantee-rewards}, 
    for any $\delta \in (0,1)$, the following holds with probability at least $1-\delta/2$.
    \begin{align*}
        % &\sum_{t=1}^T \mathop{\E}_{\pi^{c_t}_t, P^{c_t}_\star} \Bigg[\sum_{h=0}^{H-1}\left(\hat{f}_t(c_t,s_h,a_h) - f_\star(c_t,s_h,a_h)\right)^2 \Bigg|s_0\Bigg]
        % \\
        % &\hspace{42mm}
        \Esq
        \leq 
        \RegSQ+ 16 H \log(2/\delta).
    \end{align*}
\end{lemma}
% See~\cref{lemma:transition-to-oracles-regret-rewards-UD} in the supplementary material for full proof.

We analyze the expected regret of the log loss regression oracle in terms of the Hellinger distance. The following lemma is an immediate implication of Lemma A.14 in~\cite{foster2021statistical}.
% For more details, see~\cref{lemma:transition-to-oracle-rgret-dynamics}.
\begin{lemma}[Concentration of LLR regret w.r.t Hellinger distance]\label{lemma:llrs-regret}
    Under~\cref{assumption:dynamics-realizability,assumption:oracle-KL-reget}, 
    for any $\delta \in (0,1)$, with probability at least $1-\delta/2$ it holds that
    \begin{align*}
    %     &\sum_{t=1}^T 
    %     \mathop{\E}_{\pi^{c_t}_t, P^{c_t}_\star} \Bigg[
    %     \sum_{h=0}^{H-1}D^2_H(P^{c_t}_\star(\cdot|s_{h},a_{h}),\widehat{P}^{c_t}_t(\cdot|s_{h},a_{h}))
    %     \Bigg| s_0\Bigg]
    %     \\
    %     &\hspace{40mm}
        \Ehel \leq  \RegLL + 2H \log(2H/\delta).  
    \end{align*}
\end{lemma}
See~\cref{Appendix-subsec:oracle-bounds} for full proofs.

\subsection{Value-Difference Bounds}\label{subsec:val-diff-bounds-main}
In this subsection we bound the three value difference sums (\cref{pi-star-both-models,pi-star-pi-t,pi-t-both-models}).
For that purpose, we represent the value function in terms of the \emph{occupancy measures} \citep{zimin2013online}.
Recall that occupancy measures are defined as follows. For any non-contextual policy $\pi$ and dynamics $P$,
let $q_h(s,a \mid \pi, P)$ denote the probability of reaching state $s\in S$ and performing action $a\in A$ at time $h \in [H]$ of an episode generated using  policy $\pi$ and dynamics $P$. 
Thus, the value function of any policy $\pi$ with respect to the MDP $(S,A,P,r,s_0,H)$ 
% at any stage $h$ and state $s \in S$
can be presented using the occupancy measures as follows.
\begin{equation}\label{eq:val-with-occ}
    V^{\pi}_{M} (s_0)
    = 
    \sum_{h=0}^{H-1}
    \sum_{s \in S}
    \sum_{a \in A}
    q_h(s,a \mid \pi, P)\cdot r(s, a)
    . 
\end{equation}
In the analysis, we use the following notations, to denote a specific occupancy measure of interest. For all $(s,a,h) \in S \times A \times [H]$ we denote:
\begin{enumerate}[label=(\roman*),leftmargin=*]
    \item $\hat{q}^t_h(s,a): = q_h(s,a \mid \pi^{c_t}_t,\widehat{P}^{c_t}_t)$;
    \item $q^t_h(s,a):= q_h(s,a \mid \pi^{c_t}_t, P^{c_t}_\star)$;
    \item $\hat{q}^{t,\star}_h(s,a):= q_h(s,a \mid \pi^{c_t}_\star, \widehat{P}^{c_t}_t)$.
\end{enumerate}
% (1) $\hat{q}^t_h(s,a): = q_h(s,a| \pi^{c_t}_t,\widehat{P}^{c_t}_t)$, (2) $q^t_h(s,a):= q_h(s,a|\pi^{c_t}_t, P^{c_t}_\star)$ and (3) $\hat{q}^{t,\star}_h(s,a):= q_h(s,a|\pi^{c_t}_\star, \widehat{P}^{c_t}_t)$.

We now have all the required tools to state our cumulative value difference bounds.
We first bound the value difference caused by the CMDP approximation error, for the true optimal context-dependent policy $\pi_\star$.
%\ABC{(full proof in ???)}
\begin{lemma}[The cost of approximation for $\pi_\star$]\label{lemma:val-diff-1-bound}
    The following holds for any choice of parameter $\hat{\gamma}>0$. 
    \begingroup\allowdisplaybreaks
    \begin{align*}
        &\sum_{t=1}^T
        V^{\pi^{c_t}_\star}_{\M(c_t)}(s_0)
        -
        V^{\pi^{c_t}_\star}_{\Mhat_t(c_t)}(s_0)
        \\
        &\hspace{15mm} 
        \leq 
        \sum_{t=1}^T \sum_{h=0}^{H-1}\sum_{s \in S}\sum_{a \in A}
        \frac{\hat{q}^{t,\star}_h(s,a)}{\hat{\gamma} \cdot \hat{q}^t_h(s,a)} 
        \\
        &\hspace{25mm}
        +
        2
        \hat{\gamma} \cdot \Esq
        % \sum_{t=1}^T \mathop{\E}_{\pi^{c_t}_t, P^{c_t}_\star} \Bigg[\sum_{h=0}^{H-1}(\hat{f}_t(c_t,s_h,a_h) - f_\star(c_t,s_h,a_h))^2 \Big|s_0 \Bigg]
        % \\
        % & 
        +
        29
        \hat{\gamma} H^4 \cdot \Ehel
        % \sum_{t=1}^T 
        %  \mathop{\E}_{\pi^{c_t}_t, P^{c_t}_\star} \Bigg[
        % \sum_{h=0}^{H-1}
        %  D^2_H(P^{c_t}_\star(\cdot|s_{h},a_{h}), \widehat{P}^{c_t}_t(\cdot|s_{h},a_{h}))
        % \Bigg| s_0\Bigg]
        .
    \end{align*}
    \endgroup
\end{lemma}
\begin{proof}[sketch]
    Fix $t \in [T]$ and apply the value difference lemma (\cref{lemma:val-diff-efroni} in the Appendix) to obtain 
    \begin{align*}
        &V^{\pi^{c_t}_\star}_{\M(c_t)}(s_0)
        -V^{\pi^{c_t}_\star}_{\Mhat_t(c_t)}(s_0)
        \\
        &\hspace{3mm}\leq
        \sum_{h,s,a}\hat{q}^{t,\star}_h(s,a) \cdot (f_{\star}(c_t,s,a) - \hat{f}_t(c_t,s,a))
        \\
        &\hspace{4mm}+ H\sum_{h,s,a}\hat{q}^{t,\star}_h(s,a) \cdot \|P^{c_t}_\star(\cdot|s,a) - \widehat{P}^{c_t}_t(\cdot|s,a)\|_1
        .
    \end{align*}
    We then multiply each term in both sums by $\sqrt{\frac{\hat{\gamma} \cdot \hat{q}^{t}_h(s,a)}{\hat{\gamma} \cdot \hat{q}^{t}_h(s,a)}}$ and apply the arithmetic-geometric means (AM-GM) inequality to change the occupancy measure form $\hat{q}^{t,\star}$ to $\hat{q}^t$ and add a dependency in $\hat{\gamma}$. We obtain that the previous is bounded as
    \begin{align*}
        &\leq
        \sum_{h,s,a} \frac{\hat{q}^{t,\star}_h(s,a)}{\hat{\gamma} \cdot \hat{q}^t_h(s,a)}
        \\
        &\hspace{5mm}+ \frac{\hat{\gamma}}{2}\sum_{h,s,a}\hat{q}^{t}_h(s,a)\cdot(f_{\star}(c_t,s,a) - \hat{f}_t(c_t,s,a))^2
        \\
        &\hspace{5mm} +\frac{\hat{\gamma} H^2}{2}\sum_{h,s,a}\hat{q}^{t}_h(s,a)\cdot\|P^{c_t}_\star(\cdot|s,a) - \widehat{P}^{c_t}_t(\cdot|s,a)\|^2_1
        .
    \end{align*}
    Lastly we apply~\cref{eq:l-1-hellinger} to bound the squared $\ell_1$ norm with the squared Hellinger distance, 
    and then the occupancy measure change (\cref{corl:occupancy-change-of-measure} in the Appendix) to replace $\hat{q}^t$ with $q^t$. We obtain that the latter is bounded by
    \begin{align*}
        \leq 
        \sum_{h,s,a} \frac{\hat{q}^{t,\star}_h(s,a)}{\hat{\gamma} \cdot \hat{q}^t_h(s,a)} 
        +
        2
        \hat{\gamma} \cdot \Esq
        +
        29
        \hat{\gamma} H^4 \cdot \Ehel.
    \end{align*}
    The lemma follows by summing over each $t\in [T]$. For more details see~\cref{lemma:bound-of-term-1-UD,corl:bound-of-term-1-UD} in the Appendix.
\end{proof}
Next, we bound the cumulative value difference between $\pi_\star$ and $\pi_t$ on the approximated model in round $t$.
%We derive the bound using first order optimization conditions. 
%The following is an immediate corollary of~\cref{lemma:bound-of-term-2-UD}.
\begin{lemma}[Suboptimality of $\pi_\star$ in 
$\Mhat_t$]\label{lemma:comulative-bound-term-2-UD}
    It holds that
    \begin{align*}
        &\sum_{t=1}^T
        V^{\pi^{c_t}_\star}_{\Mhat_t(c_t)}(s_0)
        - 
        V^{\pi^{c_t}_t}_{\Mhat_t(c_t)}(s_0)
        \\
        &\hspace{15mm}\leq  
        \frac{TH|S||A|}{\gamma}- \sum_{t=1}^T\sum_{h=0}^{H-1}\sum_{s \in S}\sum_{a \in A} \frac{\hat{q}^{t,\star}_h(s,a)}{\gamma \cdot \hat{q}^t_h(s,a)} 
         .
    \end{align*}
\end{lemma}

\begin{proof}
    For every round $t \in [T]$, consider the first order derivative of the concave objective function in~\cref{eq:max-problem-unknown-dynamics} and denote it by $\hat{L}_t^\prime(q;c_t)$. i.e.,
    \begin{align*}
         \hat{L}_t^\prime(q;c_t) = \sum_{h=0}^{H-1}\sum_{s \in S}\sum_{a \in A}\left(\hat{f}_t(c_t,s,a) + \frac{1}{\gamma \cdot q_h(s,a)} \right).
    \end{align*}
    %Also, recall that the objective of the maximization problem (\cref{eq:max-problem-unknown-dynamics}) is a concave function.
    Let $\pi_\star = (\pi^c_\star)_{c \in \C}$ denote an optimal context-dependent policy for the true CMDP. For every round $t$, the occupancy measure $\hat{q}^{t,\star}_h(s,a):= q_h(s,a|\pi^{c_t}_\star, \widehat{P}^{c_t}_t)$ is a feasible solution for the maximization problem in~\cref{eq:max-problem-unknown-dynamics}, since $\hat{q}^{t,\star} \in \mu(\widehat{P}^{c_t}_t)$. 
    Since $\hat{q}^t$ is the optimal solution, by first order optimality conditions for concave functions \citep*{boyd2004convex} it holds that
    \begin{align*}
         &\sum_{h=0}^{H-1}\sum_{s ,a}\hat{q}^{t,\star}_h(s,a) \cdot \left(\hat{f}_t(c_t,s,a) + \frac{1}{\gamma \hat{q}^t_h(s,a)} \right)  
         \\
         &\quad-
        \sum_{h=0}^{H-1}\sum_{s ,a}\hat{q}^t_h(s,a) \cdot \left(\hat{f}_t(c_t,s,a) + \frac{1}{\gamma \hat{q}^t_h(s,a)} \right)
        \leq 0,
    \end{align*}
    which implies that
    % \begin{equation}\label{eq:first-order-bound}
    %     \begin{split}
    %     &\sum_{h=0}^{H-1}\sum_{s \in S}\sum_{a \in A} \hat{q}^{t,\star}_h(s,a) \cdot \left( \hat{f}_t(c_t,s,a) + \frac{1}{\gamma \cdot \hat{q}^t_h(s,a)}\right) 
    %     \\
    %     &\quad-\sum_{h=0}^{H-1}\sum_{s \in S}\sum_{a \in A}\hat{q}^t_h(s,a) \cdot \hat{f}_t(c_t,s,a) -\frac{H|S||A|}{\gamma} \leq 0
    %      .
    %     \end{split}
    % \end{equation}
    \begin{equation}\label{eq:first-order-bound}
        \begin{split}
        &\sum_{h=0}^{H-1}\sum_{s \in S}\sum_{a \in A} (\hat{q}^{t,\star}_h(s,a)- \hat{q}^t_h(s,a))\hat{f}_t(c_t,s,a)
        \\
        &\le \frac{H|S||A|}{\gamma} - \sum_{h=0}^{H-1}\sum_{s \in S}\sum_{a \in A}    \frac{\hat{q}^{t,\star}_h(s,a)}{\gamma \cdot \hat{q}^t_h(s,a)}
         .
        \end{split}
    \end{equation}
    % By changing sides and rearranging, we obtain that
    % \begin{equation}\label{eq:first-order-bound}
    %     \begin{split}
    %     &\sum_{h=0}^{H-1}\sum_{s \in S}\sum_{a \in A} (\hat{q}^{t,\star}_h(s,a) -\hat{q}^t_h(s,a))\cdot  \hat{f}_t(c_t,s,a)
    %     \\
    %      &\hspace{10mm}\leq \frac{H|S||A|}{\gamma} - \sum_{h=0}^{H-1}\sum_{s \in S}\sum_{a \in A} \frac{\hat{q}^{t,\star}_h(s,a)}{\gamma \cdot \hat{q}^t_h(s,a)}. 
    %     \end{split}
    % \end{equation}
    By the value representation using occupancy measures~(\cref{eq:val-with-occ}), for every round $t \in [T]$ it holds that,
    \begin{equation}\label{eq:val-by-q-term-2}
        \begin{split}
            &V^{\pi^{c_t}_\star}_{\Mhat_t(c_t)}(s_0)
            -
            V^{\pi^{c_t}_t}_{\Mhat_t(c_t)}(s_0)
            \\
            &= 
            \sum_{h=0}^{H-1}\sum_{s \in S}\sum_{a \in A} (\hat{q}^{t,\star}_h(s,a) - \hat{q}^t_h(s,a)) \cdot  \hat{f}_t(c_t,s,a)
            .
        \end{split}
    \end{equation}
    Hence, the lemma follows by combining~\cref{eq:first-order-bound,eq:val-by-q-term-2} and summing over each round $t \in [T]$.
\end{proof}
Note that for the choice in $\hat{\gamma} = \gamma$ for~\cref{lemma:val-diff-1-bound}, the term $\sum_{t=1}^T \sum_{h=0}^{H-1}\sum_{s \in S}\sum_{a \in A}
\frac{\hat{q}^{t,\star}_h(s,a)}{\hat{\gamma} \cdot \hat{q}^t_h(s,a)} $ is canceled with the second term in RHS of~\cref{lemma:comulative-bound-term-2-UD}.

Lastly, we bound the value difference caused by the approximation, for the selected policy $\pi_t$.
\begin{lemma}[The cost of approximation for $\pi_t$]\label{lemma:term-3-UD-main}
    The following holds for any two parameters $p_1, p_2 >0$.
    \begin{align*}
        &\sum_{t=1}^T
        V^{\pi^{c_t}_t}_{\Mhat_t(c_t)}(s_0)
        - 
        V^{\pi^{c_t}_t}_{\M(c_t)}(s_0)
        \\
        &\hspace{10mm}\leq 
        \frac{TH}{2 p_1} +\frac{p_1}{2} \cdot \Esq 
        %\\
        % & +
        % \frac{p_1}{2} \cdot
        % \sum_{t=1}^T\mathop{\E}_{\pi^{c_t}_t, P^{c_t}_\star} \Bigg[ \sum_{h=0}^{H-1}
        % \left(\hat{f}_t(c_t,s_h,a_h) - f_\star(c_t,s_h,a_h)\right)^2
        % \Bigg|s_0 \Bigg]
        % \\
        % & \hspace{20mm}
        + \frac{TH}{2 p_2} + 2p_2 \cdot \Ehel
        % 2p_2 \cdot
        % \sum_{t=1}^T 
        % \mathop{\E}_{\pi^{c_t}_t, P^{c_t}_\star} \Bigg[
        % \sum_{h=0}^{H-1}
        %  D^2_H(P^{c_t}_\star(\cdot|s_{h},a_{h}), \widehat{P}^{c_t}_t(\cdot|s_{h},a_{h}))
        % \Bigg| s_0\Bigg]
        .
    \end{align*}
\end{lemma}
To prove the lemma, we use the value difference lemma and AM-GM with the parameters $p_1, p_2$ similarly to showm for~\cref{lemma:val-diff-1-bound}.
% \begin{proof}[sketch]
%     The lemma follows by applying the standard value difference lemma, then AM-GM to measure the differences between the dynamics and rewards in terms of the squared $L_1$ and $L_2$ norms, respectively. Lastly we use the Hellinger distance property in~\cref{eq:l-1-hellinger} to bound the squared $L_1$ norm with the squared Hellinger distance. For more details, see~\cref{corl:term-3-UD}.
% \end{proof}

\subsection{Regret Bound}\label{subsec:regret-bound-main}
We obtain~\cref{thm:regret-main} by combining the results of~\cref{lemma:val-diff-1-bound,lemma:term-3-UD-main,lemma:comulative-bound-term-2-UD} and applying our concentration bounds stated in~\cref{lemma:lsro-regret,lemma:llrs-regret}. See proof sketch bellow. (For detailed proof see~\cref{Appendix-subsec:regret}).

\begin{proof}[sketch of~\cref{thm:regret-main}]
    % Assume the good events of~\cref{lemma:lsro-regret,lemma:llrs-regret} hold and consider the following.
    We start with the regret decomposition to the three sums. 
    \begingroup\allowdisplaybreaks
    \begin{align*}
        &\Regrv_T(\text{OMG-CMDP!})
        % \\
        % &\quad = 
        % \sum_{t=1}^T V^{\pi^{c_t}_\star}_{\M(c_t)}(s_0)
        % - V^{\pi^{c_t}_t}_{\M(c_t)}(s_0)
        \\
        &\quad = 
        \sum_{t=1}^T 
        V^{\pi^{c_t}_\star}_{\M(c_t)}(s_0)
        -
        V^{\pi^{c_t}_\star}_{\Mhat_t(c_t)}(s_0)
        \\
        &\quad\quad 
        +
        \sum_{t=1}^T 
        V^{\pi^{c_t}_\star}_{\Mhat_t(c_t)}(s_0)
        -
        V^{\pi^{c_t}_t}_{\Mhat_t(c_t)}(s_0)
        \\
        &\quad\quad+
        \sum_{t=1}^T 
        V^{\pi^{c_t}_t}_{\Mhat_t(c_t)}(s_0)
        - 
        V^{\pi^{c_t}_t}_{\M(c_t)}(s_0)
        .
   \end{align*}
    We bound the the first term by~\cref{lemma:val-diff-1-bound}, the second by~\cref{lemma:comulative-bound-term-2-UD} and the last by~\cref{lemma:term-3-UD-main}. This yields the following bound.
    \begin{align*}
        &\Regrv_T \leq 
        2
        \gamma \cdot \Esq
        % \sum_{t=1}^T \mathop{\E}_{\pi^{c_t}_t, P^{c_t}_\star} \Bigg[\sum_{h=0}^{H-1}(\hat{f}_t(c_t,s_h,a_h) - f_\star(c_t,s_h,a_h))^2 \Big|s_0 \Bigg]
        % \\ 
        % &
        +
        29
        \gamma H^4 \cdot \Ehel
        % \sum_{t=1}^T 
        %  \mathop{\E}_{\pi^{c_t}_t, P^{c_t}_\star} \Bigg[
        % \sum_{h=0}^{H-1}
        %  D^2_H(P^{c_t}_\star(\cdot|s_{h},a_{h}), \widehat{P}^{c_t}_t(\cdot|s_{h},a_{h}))
        % \Bigg| s_0\Bigg] 
        % \\
        % & 
        +
        \frac{H|S||A|T}{\gamma}
        \\
        &\quad\quad + 
        \frac{TH}{2 p_1} +
        \frac{p_1}{2} \cdot \Esq
        % \sum_{t=1}^T \mathop{\E}_{\pi^{c_t}_t, P^{c_t}_\star} \Bigg[ \sum_{h=0}^{H-1}
        % \left(\hat{f}_t(c_t,s_h,a_h) - f_\star(c_t,s_h,a_h)\right)^2
        % \Bigg|s_0 \Bigg]
        % \\
        % & 
        +
        \frac{TH}{2 p_2}
        +
        2{p_2} \cdot \Ehel
        % \sum_{t=1}^T 
        % \mathop{\E}_{\pi^{c_t}_t, P^{c_t}_\star} \Bigg[
        % \sum_{h=0}^{H-1}
        %  D^2_H(P^{c_t}_\star(\cdot|s_{h},a_{h}), \widehat{P}^{c_t}_t(\cdot|s_{h},a_{h}))
        % \Bigg| s_0\Bigg]
        .
    \end{align*}
        By~\cref{lemma:lsro-regret,lemma:llrs-regret}, the following holds with probability at least $1-\delta$.
    \begin{align*}
        &\Regrv_T \leq 
        2
        \gamma
        \left( 2\cdot
        %\Regrv_{TH}(\OLSR^\F)
        \RegSQ 
        + 16 H \log(2/\delta)\right)
        \\
        &\quad\quad +
        29\gamma H^4 
        \left( 
        %\Regrv_{TH}(\OLLR^{\Fp})
        \RegLL 
        + 2H \log(2H/\delta) \right)
        % \\
        % & 
        \\
        &\quad\quad+
        \frac{H|S||A|T}{\gamma}
        \\
        &\quad\quad + 
        \frac{TH}{2 p_1}
         +
        \frac{p_1}{2}
        \left( 2\RegSQ + 16 H \log(2/\delta) \right)
        \\
        &\quad\quad +
        \frac{TH}{2 p_2}
        +
        2{p_2}\left( \RegLL + 2H \log(2H/\delta)\right)
        \\
        &\quad =
        \widetilde{O}\left(H^{2.5} \sqrt{ T|S||A| 
        \left( \RegSQ + \RegLL + H \log \delta^{-1} \right)}\right)
    \end{align*}
    \endgroup
    where the last identity is by our choice of $\gamma$, and for $\hat{\gamma} =\gamma$, $p_1 = \sqrt{\frac{TH}{2\RegSQ+ 16 H \log(2/\delta)}}$ 
    and $p_2 = \sqrt{\frac{TH}{4\left( \RegLL + 2H \log(2H/\delta)\right)}}$. 
    Since both good events hold with probability at least $1-\delta$ we obtain the theorem.
\end{proof}

\section{Approximated Solution}
The objective of the optimization problem in \cref{eq:max-problem-unknown-dynamics} is a sum of a self-concordant barrier function (the log function) and a linear function. 
Hence, the optimal solution for the problem can be approximated efficiently using interior-point convex optimization algorithms such as Newton's Method. 
These algorithms return an $\epsilon$-approximated solution and have a running time of $O(\poly(d)\log \epsilon^{-1})$, where $d = H \abs{S}\abs{A}$ is the dimension of the problem
%To obtain $\epsilon$-approximated solution we need to run an optimization algorithm for $O(\frac{\gamma}{\epsilon})$ iterations 
\citep{nesterov1994interior}.

Suppose that in each round $t$ we derive the policy $\pi^{c_t}_t$ using, instead of the optimal solution, an occupancy measure $\hat{q}^t$ that yields an $\epsilon$-approximation to the objective of the optimization problem in \cref{eq:max-problem-unknown-dynamics}.
The following analysis shows that for $\epsilon = \frac{1}{16 \gamma T}$, we obtain a similar regret guarantee. In addition, by our choice of $\gamma$, the running time complexity of the optimization algorithm is $\poly(|S|,|A|,H, \log(T))$.
We start by bounding the difference between the optimal and the approximated iterates. (See proof in \cref{lemma:approx-iterates} in the Appendix.)
\begin{lemma}[Iterates' difference]\label{lemma:approx-iterates-main}
    For every round $t $ let 
    \textbf{\begin{align*}
        \hat{L}_t(q;c_t) 
        = &
        \sum_{h=0}^{H-1}\sum_{s \in S}\sum_{a \in A}q_h(s,a)\cdot \hat{f}_t(c_t,s,a) 
        \\
        &\hspace{15mm} +
        \frac{1}{\gamma} \sum_{h=0}^{H-1}\sum_{s \in S}\sum_{a \in A} \log(q_h(s,a)).
    \end{align*}}
    Denote the objective of the optimization problem in~\cref{eq:max-problem-unknown-dynamics}.
    Let $\widetilde{q} \in \arg\max_{q \in \mu(\widehat{P}^{c_t}_t)}\hat{L}_t(q;c_t)$.
    Let $q \in \mu(\widehat{P}^{c_t}_t)$ and suppose that $  \hat{L}_t(\widetilde{q};c_t) -\hat{L}_t(q;c_t)\leq \epsilon$. Then,
    \begin{align*}
    \sum_{s \in S} \sum_{a \in A} \sum_{h=0}^{H-1}
    \left( \frac{q_h(s,a)}{\widetilde{q}_h(s,a)} -1\right)^2
    \leq 4 \epsilon  \gamma
    .
\end{align*}
\end{lemma}
Using~\cref{lemma:approx-iterates-main}, we modify the bound of~\cref{lemma:comulative-bound-term-2-UD} and obtain the following corollary. (See~\cref{lemma:bound-of-term-2-UD-approx} in the Appendix for full proof.)
\begin{corollary}\label{corl:comulative-bound-of-term-2-UD-approx}
    For every round $t \in [T]$ and a context $c_t \in \C$, let $\widetilde{q}^t$ be the optimal solution to the maximization problem in~\cref{eq:max-problem-unknown-dynamics}. Suppose that $\hat{q}^t_h \in \mu(\widehat{P}^{c_t}_t)$ satisfies 
    $\hat{L}_t(\widetilde{q}^t;c_t) -\hat{L}_t(\hat{q}^t;c_t)\leq \epsilon
    ,
    $
    and $\epsilon \gamma \le 1/ 16$.
    Then, 
    \begin{align*}
        &\sum_{t=1}^TV^{\pi^{c_t}_\star}_{\Mhat_t(c_t)}(s_0)
        - 
        V^{\pi^{c_t}_t}_{\Mhat_t(c_t)}(s_0)
        \leq  
        \frac{H|S||A|T}{\gamma}
        \\
        &\hspace{10mm} - \sum_{t=1}^T\sum_{h=0}^{H-1}\sum_{s \in S}\sum_{a \in A} \frac{\hat{q}^{t,\star}_h(s,a)}{2\gamma \cdot \hat{q}^t_h(s,a)} 
         +
         2 T\sqrt{\epsilon \gamma H}
         .
    \end{align*}
    % where $\hat{q}^t_h(s,a) := q_h(s,a|\pi^{c_t}_t,\widehat{P}^{c_t}_t)$,
    % and $\hat{q}^{t,\star}_h(s,a):= q_h(s,a|\pi^{c_t}_\star, \widehat{P}^{c_t}_t)$, $\Mhat_t (c) := (S,A,\widehat{P}^c_t, \hat{f}_t(c,\cdot,\cdot),s_0, H)$ is the estimated CMDP at round $t$, $\pi^{c_t}_t$ is the policy induced by $\widehat{q}^t$, and $\pi_\star = (\pi^c_\star)_{c \in \C}$ is the optimal context-dependent policy of the true CMDP.
\end{corollary}
By replacing~\cref{lemma:comulative-bound-term-2-UD} with~\cref{corl:comulative-bound-of-term-2-UD-approx} we derive the following regret bound using similar steps to those shown in~\cref{subsec:regret-bound-main}. (See~\cref{Appendix-subsec:approx-sol}  for full proofs.)
\begin{theorem}[Regret bound]\label{thm:regret-approx-main}
    For any $\delta \in (0,1)$, let \\$\gamma = \sqrt{\frac{|S||A|T}{62H^3 \left( 2\RegSQ + \RegLL + 18H \log(2H/\delta) \right) }}$. Suppose that at each round $t$ we have an $\epsilon$-approximation to the optimal solution of~\cref{eq:max-problem-unknown-dynamics} for $\epsilon = \frac{1}{16 \gamma T}$.
    %such that $\epsilon \cdot \gamma \leq 1/16$.
    Then, with probability at least $1-\delta$, $\Regrv_T(\text{Approx OMG-CMDP!})$ is bounded as 
    %bounded as
    \begin{align*}
        % &\Regrv_T 
        % \leq
        % \\
        % &
        \widetilde{O}\left(H^{2.5} \sqrt{ T|S||A| 
        \left( 
        \RegSQ + \RegLL + H \log \delta^{-1} \right)}\right)
        .
    \end{align*}
\end{theorem}

\section{Discussion}
In this paper we provide the first efficient reduction from Adversarial CMDPs to online regression. The novelty of our approach is the use of concave optimization with log-barrier regularization over occupancy measures. This technique might prove useful in other settings of function approximation with a small underlying state space, e.g., block or rich observation MDPs.
% We believe the our approach can be applies to other related settings such as block MDPs, MDPs with rich observations and POMDPs.
We note that there is an $H^2$ gap between our regret upper bound and the lower bound of \citealp{levy2022optimism}. We leave closing this gap as an open problem for future research.

\section*{Acknowledgements}

AC is supported by the Israeli Science Foundation (ISF) grant no. 2250/22.\\
This project has received funding from the European Research Council (ERC) under the European Union’s Horizon 2020 research and innovation program (grant agreement No. 882396), by the Israeli Science Foundation (ISF) grant numbers 993/17 and 2549/19, Tel Aviv University Center for AI and Data Science (TAD), the Yandex Initiative for Machine Learning at Tel Aviv University, 
the Len
Blavatnik and the Blavatnik Family foundation,
and by the Israeli VATAT data science scholarship.

 %\clearpage
\bibliography{references}

\ificml
\bibliographystyle{icml2023}
\fi

\newpage
\appendix

\ificml
\onecolumn
\fi

\section{Proofs}
In the following analysis, we represent the value function in terms of the \emph{occupancy measures.} (See e.g.,~\citet{puterman2014markov,zimin2013online}).
The occupancy measures are defined as follows. For any non-contextual policy $\pi$ and dynamics $P$,
let $q_h(s,a | \pi, P)$ denote the probability of reaching state $s\in S$ and performing action $a\in A$ at time $h \in [H]$ of an episode generated using  policy $\pi$ and dynamics $P$. 
Thus, the value function of any policy $\pi$ with respect to the MDP $(S,A,P,r,s_0,H)$ 
% at any stage $h$ and state $s \in S$ 
can be presented as follows.
\begin{align}
\label{eq:value-occupancy-representation}
    V^{\pi}_{M} (s_0)
    = 
    \sum_{h=0}^{H-1}
    \sum_{s \in S}
    \sum_{a \in A}
    q_h(s,a | \pi, P)\cdot r(s, a)
    .
\end{align}
Throughout the analysis we consider the cumulative error caused by the dynamics approximation in terms of the Squared Hellinger distance between the true and approximated dynamics.
\begin{definition*}[Squared Hellinger Distance,~\cref{def:hellinger-main}]
    For any two distributions $\mathbb{P}$, $\mathbb{Q}$ over a discrete support $X$ we define the Squared Hellinger Distance as
    $$
        D^2_H(\mathbb{P}, \mathbb{Q}) 
        := 
        \sum_{x \in X} \left(\sqrt{\mathbb{P}(x)} - \sqrt{\mathbb{Q}(x)}\right)^2
        .
    $$
\end{definition*}
A useful property of the squared Hellinger distance is that for any two distributions $\mathbb{P}$ and $\mathbb{Q}$ it holds that $\|\mathbb{P} - \mathbb{Q}\|^2_1 \leq 4 D^2_H(\mathbb{P}, \mathbb{Q})$.

\paragraph{Notations.} In the analysis, we use the following notations, to denote the following occupancy measures. For all $(s,a,h) \in S \times A \times [H]$ we denote
\begin{itemize}
    \item $\hat{q}^t_h(s,a): = q_h(s,a| \pi^{c_t}_t,\widehat{P}^{c_t}_t)$,
    \item $q^t_h(s,a):= q_h(s,a|\pi^{c_t}_t, P^{c_t}_\star)$,
    \item $\hat{q}^{t,\star}_h(s,a):= q_h(s,a|\pi^{c_t}_\star, \widehat{P}^{c_t}_t)$.
\end{itemize}

\subsection{Oracle Concentration Bounds}\label{Appendix-subsec:oracle-bounds}

The following states our concentration bounds, in terms of our regression oracles regret.

\subsubsection{Least Squares Regression Oracle for Rewards Approximation}
In the following, we use Freedman’s concentration inequality.
\begin{lemma}[Freedman’s inequality (see e.g.,~\citealp{agarwal2014taming,cohen2021minimax})]\label{lemma:friedmans-ineq}
Let $\{Z_t\}_{t \geq 1}$ be a real-valued
martingale difference sequence adapted to a filtration $\{F_t\}_{t \geq 0}$ and let $\E_t[\cdot]:= \E[\cdot|F_t]$. If $|Z_t| \leq R$ almost
surely, then for any $T \in \N$ and $\eta \in (0,1/R)$ it holds with probability at least $1-\delta$ that,  
    \begin{align*}
        \sum_{t=1}^T Z_t \leq \eta \sum_{t=1}^T \E_{t-1}[Z^2_t] + \frac{ \log(1/\delta)}{\eta}.
    \end{align*}
\end{lemma}

\begin{lemma*}[concentration of LSR regret, restatement of~\cref{lemma:lsro-regret}]\label{lemma:transition-to-oracles-regret-rewards-UD}
   Under~\cref{assumption:rewards-realizability} and~\cref{assumption:oracle-gurantee-rewards},
    for any $\delta \in (0,1)$, the following holds with probability at least $1-\delta$.
    % \begin{align*}
    %     \sum_{t=1}^T \sum_{h=0}^{H-1}\sum_{s \in S}\sum_{a \in A} q^t_h(s,a) \cdot \left(\hat{f}_t(c_t,s,a) - f_\star(c_t,s,a)\right)^2
    %     \leq
    %     2\cdot
    %     \Regrv_{TH}(\OLSR^\F)
    %     + 16 H\log(1/\delta),
    % \end{align*}
    \begin{align*}
        \sum_{t=1}^T \mathop{\E}_{\pi^{c_t}_t, P^{c_t}_\star} \Bigg[\sum_{h=0}^{H-1}\left(\hat{f}_t(c_t,s_h,a_h) - f_\star(c_t,s_h,a_h)\right)^2 \Bigg|s_0\Bigg]
        \leq
        2\cdot
        \Regrv_{TH}(\OLSR^\F)
        + 16 H \log(1/\delta).
    \end{align*}
\end{lemma*}

\begin{proof}
%[proof of~\cref{lemma:transition-to-oracles-regret-rewards} under~\cref{assumption:oracle-gurantee-rewards}]
    Let us define a filtration $F_{t-1} = (\sigma^1,\ldots,\sigma^{t-1},c_t)$. Then,
    \begin{align*}
        Z_t 
        = & \E\left[\sum_{h=0}^{H-1}(\hat{f}_t(c_t,s^t_h,a^t_h)- r^t_h)^2 - (f_\star(c_t,s^t_h,a^t_h) - r^t_h)^2\Big|F_{t-1}\right]
        \\
        & -
        \sum_{h=0}^{H-1}(\hat{f}_t(c_t,s^t_h,a^t_h)- r^t_h)^2 - (f_\star(c_t,s^t_h,a^t_h) - r^t_h)^2
    \end{align*}
    defines a martingale difference sequence for that filtration.
    We first prove the following auxiliary claim.
    \begin{claim}\label{clm:aux-for-freidman-assumption-1.1}
        The followings hold for all $t \in [T]$.
        \begin{enumerate}
            \item $|Z_t| \leq 2H$.
            \item $\E\left[ \sum_{h=0}^{H-1}(\hat{f}_t(c_t,s^t_h,a^t_h)- r^t_h)^2 - (f_\star(c_t,s^t_h,a^t_h) - r^t_h)^2  |F_{t-1}\right]=$\\
            $\E\left[\sum_{h=0}^{H-1}(\hat{f}_t(c_t,s^t_h,a^t_h)- f_\star(c_t,s^t_h,a^t_h))^2 \Big|F_{t-1}\right] = $\\
            $ \sum_{h=0}^{H-1}\sum_{s \in S}\sum_{a \in A} q^t_h(s,a) \cdot \left(\hat{f}_t(c_t,s,a) - f_\star(c_t,s,a)\right)^2$.
            \item $\E[Z^2_t|F_{t-1}] \leq 4 H\cdot \E\left[ \sum_{h=0}^{H-1}(\hat{f}_t(c_t,s^t_h,a^t_h)- f_\star(c_t,s^t_h,a^t_h))^2 \Big|F_{t-1}\right]$.
        \end{enumerate}
    \end{claim}
    
    \begin{proof}
        The first property is immediate.
        For the second property, we have
        \begingroup\allowdisplaybreaks
        \begin{align*}
            &\E\left[ \sum_{h=0}^{H-1}(\hat{f}_t(c_t,s^t_h,a^t_h)- r^t_h)^2 - (f_\star(c_t,s^t_h,a^t_h) - r^t_h)^2  |F_{t-1}\right]
            \\
            = &
            \E\left[ \sum_{h=0}^{H-1}(\hat{f}_t(c_t,s^t_h,a^t_h)-f_\star(c_t,s^t_h,a^t_h))( \hat{f}_t(c_t,s^t_h,a^t_h)+f_\star(c_t,s^t_h,a^t_h)-2r^t_h) |F_{t-1}\right]
            \\
            = &
            \E\left[ \sum_{h=0}^{H-1}(\hat{f}_t(c_t,s^t_h,a^t_h)-f_\star(c_t,s^t_h,a^t_h))( \hat{f}_t(c_t,s^t_h,a^t_h)+f_\star(c_t,s^t_h,a^t_h)-2\E[r^t_h | c_t, s^t_h,a^t_h])  |F_{t-1}\right]
            \\
            = &
            \E\left[ \sum_{h=0}^{H-1}(\hat{f}_t(c_t,s^t_h,a^t_h)-f_\star(c_t,s^t_h,a^t_h))( \hat{f}_t(c_t,s^t_h,a^t_h)+f_\star(c_t,s^t_h,a^t_h)-2f_\star(c_t,s^t_h,a^t_h))  |F_{t-1}\right]
            \\
            = &
            \E\left[ \sum_{h=0}^{H-1}(\hat{f}_t(c_t,s^t_h,a^t_h)-f_\star(c_t,s^t_h,a^t_h))^2 |F_{t-1}\right],
        \end{align*}
        \endgroup
        where in the second and third equalities we used that $\E[r^t_h|c_t,s^t_h,a^t_h] = f_\star(c_t,s^t_h,a^t_h)$ and that $\hat{f}_t(c_t,s^t_h,a^t_h)$ and $r^t_h$ are independent given $s^t_h,a^t_h$ and the filtration $F_{t-1}$.
        
        For the third property, consider the following derivation.
        \begingroup\allowdisplaybreaks
        \begin{align*}
            &\E[Z^2_t|F_{t-1}]
            \\
            = &
            \E\left[\left( \sum_{h=0}^{H-1}(\hat{f}_t(c_t,s^t_h,a^t_h)- r^t_h)^2 - (f_\star(c_t,s^t_h,a^t_h) - r^t_h)^2  \right)^2 \Bigg|F_{t-1}\right] 
            \\
            & -
            \E^2\left[ \sum_{h=0}^{H-1}(\hat{f}_t(c_t,s^t_h,a^t_h)- r^t_h)^2 - (f_\star(c_t,s^t_h,a^t_h) - r^t_h)^2  \Bigg|F_{t-1}\right]
            \\
            \leq &
            \E\left[\left( \sum_{h=0}^{H-1}(\hat{f}_t(c_t,s^t_h,a^t_h)- r^t_h)^2 - (f_\star(c_t,s^t_h,a^t_h) - r^t_h)^2  \right)^2 \Bigg|F_{t-1}\right]
            \\
            \leq &
            H \cdot \E\left[ \sum_{h=0}^{H-1}\left((\hat{f}_t(c_t,s^t_h,a^t_h)- r^t_h)^2 - (f_\star(c_t,s^t_h,a^t_h) - r^t_h)^2  \right)^2 \Bigg|F_{t-1}\right]
            \\ 
            = &
            H \cdot \E\left[ \sum_{h=0}^{H-1}(\hat{f}_t(c_t,s^t_h,a^t_h)-f_\star(c_t,s^t_h,a^t_h))^2( \hat{f}_t(c_t,s^t_h,a^t_h)+f_\star(c_t,s^t_h,a^t_h)-2r^t_h)^2 \Bigg|F_{t-1}\right]
            \\
            \leq &
            4H \cdot \E\left[ \sum_{h=0}^{H-1}(\hat{f}_t(c_t,s^t_h,a^t_h)-f_\star(c_t,s^t_h,a^t_h))^2 \Bigg|F_{t-1}\right].
        \end{align*}
        \endgroup
    \end{proof}
    
    We now back to the proof of the lemma.
    %~\cref{lemma:transition-to-oracles-regret-rewards-UD}.
    By~\cref{lemma:friedmans-ineq,clm:aux-for-freidman-assumption-1.1}, for $\eta \in (0,1/2H)$ with probability at least $1-\delta$ it holds that
    \begingroup\allowdisplaybreaks
    \begin{align*}
        &\sum_{t=1}^T \E\left[\sum_{h=0}^{H-1}(\hat{f}_t(c_t,s^t_h,a^t_h)- r^t_h)^2 - (f_\star(c_t,s^t_h,a^t_h) - r^t_h)^2 \Big|F_{t-1}\right]
        \\
        & -
        \sum_{h=0}^{H-1}(\hat{f}_t(c_t,s^t_h,a^t_h)- r^t_h)^2 - (f_\star(c_t,s^t_h,a^t_h) - r^t_h)^2
        \\
        = &
        \sum_{t=1}^T Z_t
        \\
        \tag{By~\cref{lemma:friedmans-ineq}}
        \leq & 
        \eta \sum_{t=1}^T \E_{t-1}[Z^2_t] + \frac{ \log(1/\delta)}{\eta}
        \\
        \tag{By~\cref{clm:aux-for-freidman-assumption-1.1}}
        \leq &
        \eta \cdot 4 H \cdot \sum_{t=1}^T \E\left[ \sum_{h=0}^{h-1}(\hat{f}_t(c_t,s^t_h,a^t_h)- f_\star(c_t,s^t_h,a^t_h))^2 \Big|F_{t-1}\right] + \frac{\log(1/\delta)}{\eta}.
    \end{align*}
    \endgroup
    The latter implies that
    \begin{align*}
        &(1-\eta \cdot 4 \cdot H) \cdot \sum_{t=1}^T \E\left[\sum_{h=0}^{H-1}(\hat{f}_t(c_t,s^t_h,a^t_h)-f_\star(c_t,s^t_h,a^t_h))^2\Big|F_{t-1}\right]
        \\
        & \leq 
        \sum_{t=1}^T\sum_{h=0}^{H-1}(\hat{f}_t(c_t,s^t_h,a^t_h)- r^t_h)^2 - (f_\star(c_t,s^t_h,a^t_h) - r^t_h)^2 + \frac{ \log(1/\delta)}{\eta}.
    \end{align*}
    For $\eta = \frac{1}{8 H} \in (0,1/2H)$ we obtain
    \begin{align*}
        &\frac{1}{2} \cdot  \sum_{t=1}^T \E\left[\sum_{h=0}^{H-1}(\hat{f}_t(c_t,s^t_h,a^t_h)-f_\star(c_t,s^t_h,a^t_h))^2\Big|F_{t-1}\right]
        \\
        & \leq 
        \sum_{t=1}^T
        \sum_{h=0}^{H-1}(\hat{f}_t(c_t,s^t_h,a^t_h)- r^t_h)^2 - (f_\star(c_t,s^t_h,a^t_h) - r^t_h)^2 + 8 H \log(1/\delta).
    \end{align*}
    Thus, when combine the above with part 2 of~\cref{clm:aux-for-freidman-assumption-1.1} we obtain,
    \begin{equation}\label{ineq:KD-by-friedman-assumption-1.1}
        \begin{split}
            &\sum_{t=1}^T \sum_{h=0}^{H-1}\sum_{s \in S}\sum_{a \in A} q^t_h(s,a) \cdot \left(\hat{f}_t(c_t,s,a) - f_\star(c_t,s,a)\right)^2
            \\
            & \leq
            2\cdot\sum_{t=1}^T
            \sum_{h=0}^{H-1}(\hat{f}_t(c_t,s^t_h,a^t_h)- r^t_h)^2 - (f_\star(c_t,s^t_h,a^t_h) - r^t_h)^2 + 16 H \log(1/\delta).  
        \end{split}
    \end{equation}
    By the oracle guarantees~(\cref{assumption:oracle-gurantee-rewards}),
    \begin{equation}\label{ineq:oracle-regret-KD-assumption-1.1}
        \begin{split}
            &\sum_{t=1}^T \sum_{h=0}^{H-1}(\hat{f}_t(c_t,s^t_h,a^t_h)- r^t_h)^2 - (f_\star(c_t,s^t_h,a^t_h) - r^t_h)^2
            \\
            & \leq 
            \sum_{t=1}^T \sum_{h=0}^{H-1} (\hat{f}_t(c_t,s^t_h,a^t_h) - r^t_h)^2 
            - \inf_{f \in \F}  \sum_{t=1}^T \sum_{h=0}^{H-1}  (f(c_t,s^t_h,a^t_h) - r^t_h)^2
            \\
            & \leq  
            \Regrv_{TH}(\OLSR^\F).
        \end{split}
    \end{equation}
    By combining~\cref{ineq:KD-by-friedman-assumption-1.1,ineq:oracle-regret-KD-assumption-1.1} we obtain the lemma as,
    \begin{align*}
        \sum_{t=1}^T \sum_{h=0}^{H-1}\sum_{s \in S}\sum_{a \in A} q^t_h(s,a) \cdot \left(\hat{f}_t(c_t,s,a) - f_\star(c_t,s,a)\right)^2
        \leq
        2\cdot
        \Regrv_{TH}(\OLSR^\F)
        + 16 H \log(1/\delta),
    \end{align*}
    which using the fact that $q^t_h(s,a):= q_h(s,a|\pi^{c_t}_t, P^{c_t}_\star)$ implies
    \begin{align*}
        \sum_{t=1}^T \mathop{\E}_{\pi^{c_t}_t, P^{c_t}_\star} \Bigg[\sum_{h=0}^{H-1}\left(\hat{f}_t(c_t,s_h,a_h) - f_\star(c_t,s_h,a_h)\right)^2 \Bigg|s_0\Bigg]
        \leq
        2\cdot
        \Regrv_{TH}(\OLSR^\F)
        + 16 H \log(1/\delta).
    \end{align*}
\end{proof}

\subsubsection{Log-Loss Regression Oracle for Dynamics Approximation}
To bound the oracle regret, we use the following lemma.
\begin{lemma}[Lemma A.14 in \citealp{foster2021statistical}]\label{lemma:A.14-dylan-LL}
    Consider a sequence of $\{0,1\}$-valued random variables $(\I_t)_{t \leq T}$ where $\I_t$ is $\mathit{F}^{(t-1)}$-measurable. For any $\delta \in (0,1)$ we have that with probability at least $1-\delta$,
    \begin{align*}
        \sum_{t=1}^T \E_{t-1}\left[ D^2_H(\widehat{g}^{(t)}(x^{(t)}), {g}^{(t)}_\star(x^{(t)})) \right]\I_t
        \leq
         \sum_{t=1}^T \left( \log^{(t)}(\widehat{g}^{(t)}) - \log^{(t)}({g}^{(t)}_\star) \right)\I_t + 2 \log(1/\delta).
    \end{align*}
\end{lemma}

\begin{lemma*}[concentration of LLR regret w.r.t Hellinger distance, restatement of~\cref{lemma:llrs-regret}]\label{lemma:transition-to-oracle-rgret-dynamics}
    Under~\cref{assumption:dynamics-realizability} and~\cref{assumption:oracle-KL-reget}, 
    for any $\delta \in (0,1)$, with probability at least $1-\delta$ it holds that
    \[
        \sum_{t=1}^T 
        \mathop{\E}_{\pi^{c_t}_t, P^{c_t}_\star} \Bigg[
        \sum_{h=0}^{H-1}D^2_H(P^{c_t}_\star(\cdot|s_{h},a_{h}),\widehat{P}^{c_t}_t(\cdot|s_{h},a_{h}))
        \Bigg| s_0\Bigg]
        \leq 
        \Regrv_{TH}(\OLLR^{\Fp}) + 2H \log(H/\delta).
    \]
\end{lemma*}

\begin{proof}
    Recall $q^t_h(s,a):= q_h(s,a|\pi^{c_t}_t, P^{c_t}_\star)$.
    \begingroup
    \allowdisplaybreaks
    \begin{align*}
        &\sum_{t=1}^T 
        \mathop{\E}_{\pi^{c_t}_t, P^{c_t}_\star} \Bigg[
        \sum_{h=0}^{H-1}D^2_H(P^{c_t}_\star(\cdot|s_{h},a_{h}),\widehat{P}^{c_t}_t(\cdot|s_{h},a_{h}))
        \Bigg| s_0\Bigg]
        \\ 
        & =
        \sum_{t=1}^T
        \sum_{h=0}^{H-1}
        \sum_{s \in S}
        \sum_{a \in A}
        q^t_h(s,a) 
        \cdot D^2_H(P^{c_t}_\star(\cdot|s,a),\widehat{P}^{c_t}_t(\cdot|s,a))
        \\
        & =
        \sum_{h=0}^{H-1}
        \sum_{t=1}^T
        \sum_{s \in S}
        \sum_{a \in A}
        q^t_h(s,a)
        \cdot D^2_H(P^{c_t}_\star(\cdot|s,a),\widehat{P}^{c_t}_t(\cdot|s,a))
        \\
        & \underbrace{=}_{(i)} 
        \sum_{h=0}^{H-1}\sum_{t=1}^T\E\left[D^2_H(P^{c_t}_\star(\cdot|s_h,a_h),\widehat{P}^{c_t}_t(\cdot|s_h,a_h)) \Big| \Hist_{t-1},c_t\right]
        \\
        \tag{By~\cref{lemma:A.14-dylan-LL}, holds w.p. at least $1-\delta$}
        & \leq
        % \sum_{h=0}^{H-1}\left(\sum_{t=1}^T\log\left(\frac{P^{c_t}_\star(s^t_{h+1}|s^t_h,a^t_h)}{\widehat{P}^{c_t}(s^t_{h+1}|s^t_h,a^t_h)}\right) + 2 \log(H/\delta)\right)
        % \\
        % & =
        \sum_{t=1}^T \sum_{h=0}^{H-1} \log\left(\frac{1}{\widehat{P}^{c_t}(s^t_{h+1}|s^t_h,a^t_h)} \right) -  \sum_{t=1}^T \sum_{h=0}^{H-1}\log\left(\frac{1}{P^{c_t}_\star(s^t_{h+1}|s^t_h,a^t_h)}\right) + 2 H\log(H/\delta)
        \\
        & \leq
        \tag{By realizability}
         \sum_{t=1}^T \sum_{h=0}^{H-1} \log\left(\frac{1}{\widehat{P}^{c_t}(s^t_{h+1}|s^t_h,a^t_h)} \right) -  \inf_{P \in \Fp}\left\{\sum_{t=1}^T \sum_{h=0}^{H-1}\log\left(\frac{1}{P^{c_t}(s^t_{h+1}|s^t_h,a^t_h)}\right)\right\} + 2 H\log(H/\delta)
        \\
        & \leq
        \Regrv_{TH}(\OLLR^{\Fp}) + 2H \log(H/\delta).
    \end{align*}
    \endgroup
    The filtration used in $(i)$ is over the history up to time $t$, $\Hist_{t-1} = (\sigma_1, \ldots,\sigma_{t-1})$ and the context in time $t$, $c_t$.
\end{proof}

\subsection{Regret Analysis}

Recall the Helligner distance given in \cref{def:hellinger-main}. The following change of measure result is due to \cite{foster2021statistical}.
\begin{lemma}[Lemma A.11 in~\citealp{foster2021statistical}]\label{lemma:A.14-dylan}
Let $\mathbb{P}$ and $\mathbb{Q}$ be two probability measures on $(\mathcal{X}, \mathrm{F})$. For all $h:\mathcal{X} \to \R $ with $0 \leq h(X) \leq R$ almost surely under $\mathbb{P}$ and $\mathbb{Q}$, we have
\begin{align*}
    \left| \E_{\mathbb{P}}[h(X)] -  \E_{\mathbb{Q}}[h(X)] \right| 
    \leq
    \sqrt{2 R ( \E_{\mathbb{P}}[h(X)] +  \E_{\mathbb{Q}}[h(X)])\cdot D^2_H(\mathbb{P}, \mathbb{Q})}
    .
\end{align*}
In particular,
\begin{align*}
    \E_{\mathbb{P}}[h(X)] 
    \leq
    3\E_{\mathbb{Q}}[h(X)] + 4RD^2_H(\mathbb{P}, \mathbb{Q})
    .
\end{align*}
\end{lemma}

Next, we need the following refinement of the previous result.
\begin{corollary}\label{corl:A.11}
    For any $\beta \geq 1$,
    \begin{align*}
    \E_{\mathbb{P}}[h(X)] 
    \leq
    (1+1/\beta)\E_{\mathbb{Q}}[h(X)] + 3\beta RD^2_H(\mathbb{P}, \mathbb{Q})
    .
\end{align*}
\end{corollary}

\begin{proof}
    Let $\eta \in (0,1)$. Consider the following derivation.
    \begingroup
    \allowdisplaybreaks
    \begin{align*}
        % &
        \E_{\mathbb{P}}[h(X)] - \E_{\mathbb{Q}}[h(X)]
        % \\
        &
        \leq
        \sqrt{2 R ( \E_{\mathbb{P}}[h(X)] +  \E_{\mathbb{Q}}[h(X)])\cdot D^2_H(\mathbb{P}, \mathbb{Q})}
        \\
        &
        \leq
        \eta (\E_{\mathbb{P}}[h(X)] + \E_{\mathbb{Q}}[h(X)]) + \frac{R}{2\eta}D^2_H(\mathbb{P}, \mathbb{Q}).
    \end{align*}
    \endgroup
    The above implies
    \begingroup\allowdisplaybreaks
    \begin{align*}
        % &
        \E_{\mathbb{P}}[h(X)] 
        % \\
        &
        \leq 
        \frac{1+\eta}{1-\eta}\E_{\mathbb{Q}}[h(X)] + \frac{R}{2\eta (1-\eta)}D^2_H(\mathbb{P},\mathbb{Q})
        \\
        \tag{Plug $\eta = \frac{1}{2\beta+1}$ for all $\beta \in (0, \infty)$.}
        &
        =
        \left(1 + \frac{1}{\beta}\right)\E_{\mathbb{Q}}[h(X)] + 3R \frac{(2\beta +1)^2}{2\beta} D^2_H(\mathbb{P},\mathbb{Q})
        \\
        \tag{For any $\beta \geq 1$}
        &
        \leq
        \left(1 + \frac{1}{\beta}\right)\E_{\mathbb{Q}}[h(X)] + 3R \beta D^2_H(\mathbb{P},\mathbb{Q}).
    \end{align*}
    \endgroup
\end{proof}

In the following regret analysis, we use the value change of measure with respect to the Hellinger distance, introduced by~\citet{levy2022counterfactul}.
\begin{lemma}[Lemma 3 in~\citealp{levy2022counterfactul}]\label{lemma:refined-val-diff-same-r-Hel}
    Let $r: S \times A \to [0,1]$ be a bounded expected rewards function. Let $P_\star$ and $\widehat{P}$ denote two dynamics and consider the MDPs $M  = (S,A,P_\star,r, s_0,H)$ and $\widehat{M}  = (S,A,\widehat{P},r, s_0,H)$.
    Then, for any policy $\pi$ we have
    \begin{align*}
        V^\pi_{\widehat{M}}(s)
        \le
        3 V^\pi_{{M}}(s)
        +
        9 H^2
        \mathop{\E}_{P_\star, \pi}
        \brk[s]*{
        \sum_{h=0}^{H-1}
        D_H^2(\widehat{P}(\cdot|s_{h},a_{h}), {P}_\star(\cdot|s_{h},a_{h}))
        \Bigg| s_{0} = s
        }
        .
    \end{align*}
\end{lemma}

\begin{proof}
    We first prove by backwards induction that for all $h \in [H-1]$ the following holds.
    \begin{align*}
        V^\pi_{\widehat{M},h}(s)
        \le
        \brk*{1+\frac1H}^{H-h}
        \brk[s]*{
        V^\pi_{{M},h}(s)
        +
        \mathop{\E}_{P_\star, \pi}
        \brk[s]*{
        \sum_{h'=h}^{H-1}
        3H^2 D_H^2(\widehat{P}(\cdot|s_{h'},a_{h'}), {P}_\star(\cdot|s_{h'},a_{h'}))
        \Bigg| s_{h} = s
        }
        }
        .
    \end{align*}
    The base case, $h=H-1$ is immediate since $V^\pi_{\widehat{M},h}(s) = V^\pi_{{M},h}(s)$.
    Now, we assume that the above holds for $h+1$ and prove that it holds for $h$.
    To see this, we have that
    \begingroup\allowdisplaybreaks
    \begin{align*}
        &
        V^\pi_{\widehat{M},h}(s)
        \tag{By Bellman's equations}
        = 
        \mathop{\E}_{a \sim \pi(\cdot|s)}
        \brk[s]*{
        r(s,a)
        +
        \E_{s' \sim \widehat{P}(\cdot|s,a)}
        \brk[s]*{V^\pi_{\widehat{M},h+1}(s') 
        } }
        \\
        \tag{\cref{corl:A.11}}
        \le &
        \mathop{\E}_{a \sim \pi(\cdot|s)}
        \brk[s]*{
        r(s,a)
        +
        \brk*{1+\frac1H}\E_{s' \sim {P}_\star(\cdot|s,a)}
        \brk[s]*{V^\pi_{\widehat{M},h+1}(s')
        }
        +
        3H^2 D_H^2(\widehat{P}(\cdot|s,a), {P}_\star(\cdot|s,a))
        } 
        \\
        \tag{Induction hypothesis}
        \le &
        \mathop{\E}_{a \sim \pi(\cdot|s)}
        \brk[s]*{
        r(s,a)
        +
        3H^2 D_H^2(\widehat{P}(\cdot|s,a), {P}_\star(\cdot|s,a))
        }
        \\
        + &
        \mathop{\E}_{a \sim \pi(\cdot|s)}
        \brk[s]*{
        \brk*{1+\frac1H}^{H-h}\mathop{\E}_{s' \sim {P}_\star(\cdot|s,a)}
        \brk[s]*{
        V^\pi_{{M},h+1}(s')
        }
        }
        \\
        + &
        \mathop{\E}_{a \sim \pi(\cdot|s)}
        \brk[s]*{
        \brk*{1+\frac1H}^{H-h}\mathop{\E}_{s' \sim {P}_\star(\cdot|s,a)}
        \brk[s]*{
        \E\brk[s]*{
        \sum_{h'=h+1}^{H-1}
        3H^2 D_H^2(\widehat{P}(\cdot|s_{h'},a_{h'}), {P}_\star(\cdot|s_{h'},a_{h'}))
        \Bigg| s_{h+1} = s'
        }
        }
        }
        \\
        \tag{$r, D_H^2 \ge 0$}
        \le &
        \brk*{1+\frac1H}^{H-h}
        \mathop{\E}_{a \sim \pi(\cdot|s)}
        \brk[s]*{
        r(s,a)
        +
        \mathop{\E}_{s' \sim {P}_\star(\cdot|s,a)}
        \brk[s]*{
        V^\pi_{{M},h+1}(s')
        }
        }
        \\
        + &
        \brk*{1+\frac1H}^{H-h}
        \mathop{\E}_{P_\star, \pi}
        \brk[s]*{
        \sum_{h'=h}^{H-1}
        3H^2 D_H^2(\widehat{P}(\cdot|s_{h'},a_{h'}), {P}_\star(\cdot|s_{h'},a_{h'}))
        \Bigg| s_{h} = s
        }
        \\
        \tag{By Bellman's equations}
        = &
        \brk*{1+\frac1H}^{H-h}
        \brk[s]*{
        V^\pi_{{M},h}(s)
        +
        \mathop{\E}_{P_\star, \pi}
        \brk[s]*{
        \sum_{h'=h}^{H-1}
        3H^2 D_H^2(\widehat{P}(\cdot|s_{h'},a_{h'}), {P}_\star(\cdot|s_{h'},a_{h'}))
        \Bigg| s_{h} = s
        }
        }
        ,
    \end{align*}
    \endgroup
    as desired. Plugging in $h=0$ and using that $\brk*{1+\frac1H}^{H} \le 3$ concludes the proof.
\end{proof}
This change of measure lemma upper bounds the value  difference caused by the use of an approximated dynamics, instead of the true one, in terms of the expected Hellinger distance across a trajectory.
This bound might seem very loose as a value difference bound, however, when the rewards are very small, it yields a significantly tighter then the standard bounds. In~\cref{lemma:rewards-true-occ-mesure,lemma:dynamics-true-occ-mesure} we apply the value change of measure lemma where the rewards function are the squared loss of the rewards approximation, and the squared $L_1$ loss of the dynamics approximation. As these rewards are very small, Lemma $3$ implies that those expected approximation errors with respect to the approximated dynamics $\widehat{P}$ are at most a small constant multiple of these expected errors where the expectation is with respect to the true dynamics $P_\star$. Thus, the lemma helps us to translated the expected errors from the approximated measures to the true measures.

The following is an immediate corollary of~\cref{lemma:refined-val-diff-same-r-Hel,eq:value-occupancy-representation}, which we will use in our analysis.
\begin{corollary}[Occupancy measures change]
\label{corl:occupancy-change-of-measure}
    For any (non-contextual) policy $\pi$, two dynamics $P$ and $\widehat{P}$, and rewards function $r$ that is bounded in $[0,1]$ it holds that
    \begin{align*}
        \sum_{h=0}^{H-1}
        \sum_{s \in S_h}
        \sum_{a \in A}
        q_h(s,a| \pi,\widehat{P})\cdot
        r(s,a)
        \leq &
        3\sum_{h=0}^{H-1}
        \sum_{s \in S_h}
        \sum_{a \in A}
        q_h(s,a| \pi,P)\cdot
        r(s,a)
        \\
        & +
        9H^2
        \sum_{h=0}^{H-1}
        \sum_{s \in S_h}
        \sum_{a \in A}
        q_h(s,a| \pi,P)\cdot
         D^2_H(P(\cdot|s,a), \widehat{P}(\cdot|s,a)) 
        .
    \end{align*}
\end{corollary}
In addition, we use the following version on the Value Difference Lemma introduced by~\citet{efroni2020optimistic}.
\begin{lemma}[Value-difference, Corollary 1;~\citealp{efroni2020optimistic}]\label{lemma:val-diff-efroni}
    Let $M$, $M'$ be any $H$-finite horizon MDPs. Then, for any two policies $\pi$, $\pi'$ the following holds
    \begingroup\allowdisplaybreaks
    \begin{align*}
        V^{\pi,M}_0(s) -  V^{\pi',M'}_0(s) 
        =&
        \sum_{h=0}^{H-1} \E [ \langle  Q^{\pi,M}_h(s_h, \cdot) , \pi_h(\cdot|s_h) - \pi'_h(\cdot|s_h) \rangle |s_0 = s, \pi',M' ]
        \\
        & +
        \sum_{h=0}^{H-1} \E [ 
        r_h(s_h,a_h) -  r'_h(s_h,a_h)
        + (p_h(\cdot|s_h,a_h)-p'_h(\cdot|s_h,a_h)) V^{\pi,M}_{h+1}
        |s_h = s, \pi',M' ] .       
    \end{align*}
    \endgroup
\end{lemma}
We are now have all the required tolls for the regret analysis.

\subsubsection{Probability Measure Transitions}
In the following, we present probability transition measures when applies to the cumulative approximation error for both the rewards and dynamics.
\begin{lemma}[probabilities transition for rewards]\label{lemma:rewards-true-occ-mesure}
    The following holds.
    \begin{align*}
        \sum_{t=1}^T \sum_{h=0}^{H-1}\sum_{s \in S}\sum_{a \in A}
        \hat{q}^t_h(s,a) \cdot (\hat{f}_t(c_t,s,a) - &f_\star(c_t,s,a))^2
        \\
        \leq & 3 \sum_{t=1}^T \mathop{\E}_{\pi^{c_t}_t, P^{c_t}_\star} \Bigg[\sum_{h=0}^{H-1}(\hat{f}_t(c_t,s_h,a_h) - f_\star(c_t,s_h,a_h))^2 \Big|s_0 \Bigg]
        \\
        & +
        9 H^2 \sum_{t=1}^{T}  
        \mathop{\E}_{\pi^{c_t}_t, P^{c_t}_\star} \Bigg[
        \sum_{h=0}^{H-1}
         D^2_H(P^{c_t}_\star(\cdot|s_{h},a_{h}), \widehat{P}^{c_t}_t(\cdot|s_{h},a_{h}))
        \Bigg| s_0\Bigg]
        .
    \end{align*}
\end{lemma}

\begin{proof}
    For any context $c \in \C$ and function $\hat{f}_t \in \F$ we have that $\tilde{r}^c(s,a) := (\hat{f}_t(c,s,a) - f_\star(c,s,a))^2$ is bounded in $[0,1]$. 
    Recall $\hat{q}^t_h(s,a) = q_h(s,a|\pi^{c_t}_t, \widehat{P}^{c_t}_t)$ and $q^t_h(s,a):= q_h(s,a|\pi^{c_t}_t, P^{c_t}_\star)$.
    Hence, by~\cref{corl:occupancy-change-of-measure}, the following holds.
    \begingroup\allowdisplaybreaks
    \begin{align*}
        \sum_{t=1}^T \sum_{h=0}^{H-1}\sum_{s \in S}\sum_{a \in A}
        \hat{q}^t_h(s,a) \cdot (\hat{f}_t(c_t,s,a) - &f_\star(c_t,s,a))^2
        \leq 
        3\sum_{t=1}^{T}\sum_{h=0}^{H-1}
        \sum_{s \in S_h}
        \sum_{a \in A}
        q^t_h(s,a)
        (\hat{f}_t(c_t,s,a) - f_\star(c_t,s,a))^2
        \\
        & 
        +
        9H^2
        \sum_{t=1}^{T}
        \sum_{h=0}^{H-1}
        \sum_{s \in S_h}
        \sum_{a \in A}
        q^t_h(s,a)\cdot
         D^2_H(P^{c_t}_\star(\cdot|s_{h},a_{h}), \widehat{P}^{c_t}_t(\cdot|s_{h},a_{h})) 
         .
    \end{align*}
    \endgroup
    Thus, the lemma follows.
\end{proof}

\begin{lemma}[probability transition for dynamics]\label{lemma:dynamics-true-occ-mesure}
    The following holds.
    \begin{align*}
        \sum_{t=1}^T \sum_{h=0}^{H-1}\sum_{s \in S}\sum_{a \in A}
        \hat{q}^t_h(s,a) \cdot& \left(\sum_{s' \in S} \left|\widehat{P}^{c_t}_t(s'|s,a) - P^{c_t}_\star(s'|s,a)\right|\right)^2
        \\
        \leq & 
        48 H^2 \sum_{t=1}^T 
         \mathop{\E}_{\pi^{c_t}_t, P^{c_t}_\star} \Bigg[
        \sum_{h=0}^{H-1}
         D^2_H(P^{c_t}_\star(\cdot|s_{h},a_{h}), \widehat{P}^{c_t}_t(\cdot|s_{h},a_{h}))
        \Bigg| s_0\Bigg]
         .
    \end{align*}
    where $\hat{q}^t_h(s,a): = q_h(s,a|\pi^{c_t}_t, \widehat{P}^{c_t}_t)$ and $q^t_h(s,a):= q_h(s,a|\pi^{c_t}_t, P^{c_t}_\star)$.
\end{lemma}

\begin{proof}
    For any context $c \in \C$ and context-dependent dynamics $\widehat{P}_t \in \Fp$ we have that 
    \[
        \tilde{r}^c(s,a) := \left(\sum_{s' \in S} \left|\widehat{P}^{c}_t(s'|s,a) - P^{c}_\star(s'|s,a)\right|\right)^2
    \]
     is bounded in $[0,4]$. Hence, by~\cref{corl:occupancy-change-of-measure}, the following holds.
    \begingroup\allowdisplaybreaks
    \begin{align*}
        &\sum_{t=1}^T \sum_{h=0}^{H-1}\sum_{s \in S}\sum_{a \in A}
        \hat{q}^t_h(s,a) \cdot \left(\sum_{s' \in S} \left|\widehat{P}^{c_t}_t(s'|s,a) - P^{c_t}_\star(s'|s,a)\right|\right)^2
        \\
        \leq & 
        3\sum_{t=1}^{T}\sum_{h=0}^{H-1}
        \sum_{s \in S_h}
        \sum_{a \in A}
        q^t_h(s,a)
        \left(\sum_{s' \in S} \left|\widehat{P}^{c_t}_t(s'|s,a) - P^{c_t}_\star(s'|s,a)\right|\right)^2
        \\
        & 
        +
        36H^2
        \sum_{t=1}^{T}
        \sum_{h=0}^{H-1}
        \sum_{s \in S_h}
        \sum_{a \in A}
        q^t_h(s,a)\cdot
         D^2_H(P^{c_t}_\star(\cdot|s_{h},a_{h}), \widehat{P}^{c_t}_t(\cdot|s_{h},a_{h})) 
         \\
         \tag{$\|\cdot\|^2_1 \leq 4D^2_H$}
        \leq & 
        12\sum_{t=1}^{T}\sum_{h=0}^{H-1}
        \sum_{s \in S_h}
        \sum_{a \in A}
        q^t_h(s,a)
        D^2_H(P^{c_t}_\star(\cdot|s_{h},a_{h}), \widehat{P}^{c_t}_t(\cdot|s_{h},a_{h})) 
        \\
        & 
        +
        36H^2
        \sum_{t=1}^{T}
        \sum_{h=0}^{H-1}
        \sum_{s \in S_h}
        \sum_{a \in A}
        q^t_h(s,a)\cdot
         D^2_H(P^{c_t}_\star(\cdot|s_{h},a_{h}), \widehat{P}^{c_t}_t(\cdot|s_{h},a_{h})) 
         \\
         \leq &
         48 H^2 \sum_{t=1}^T 
        \mathop{\E}_{\pi^{c_t}_t, P^{c_t}_\star} \Bigg[
        \sum_{h=0}^{H-1}
         D^2_H(P^{c_t}_\star(\cdot|s_{h},a_{h}), \widehat{P}^{c_t}_t(\cdot|s_{h},a_{h}))
        \Bigg| s_0\Bigg]
         .
    \end{align*}
    \endgroup
    Thus, the lemma follows.
\end{proof}

\subsubsection{Value Difference Bounds}
In the following we derive three value difference bounds, which will be used to bound the regret.
\begin{lemma}\label{lemma:bound-of-term-1-UD}
    The following holds for any $\hat{\gamma} >0$.
    \begingroup\allowdisplaybreaks
    \begin{align*}
        \sum_{t=1}^T 
        V^{\pi^{c_t}_\star}_{\M(c_t)}(s_0)
        -
        V^{\pi^{c_t}_\star}_{\Mhat_t(c_t)}(s_0)
        \leq & 
         \sum_{t=1}^T \sum_{h=0}^{H-1}\sum_{s \in S}\sum_{a \in A}
        \frac{\hat{q}^{t,\star}_h(s,a)}{\hat{\gamma} \cdot \hat{q}^t_h(s,a)} 
        \\
        & +
        \frac{\hat{\gamma}}{2}
        \sum_{t=1}^T \sum_{h=0}^{H-1}\sum_{s \in S}\sum_{a \in A}
        \hat{q}^t_h(s,a) \cdot \left(\hat{f}_t(c_t,s,a) - f_\star(c_t,s,a)\right)^2
        \\
        & +
        \frac{\hat{\gamma} \cdot H^2 }{2}
        \sum_{t=1}^T \sum_{h=0}^{H-1}\sum_{s \in S}\sum_{a \in A}
        \hat{q}^t_h(s,a) \cdot \left(\sum_{s' \in S} \left|\widehat{P}^{c_t}_t(s'|s,a) - P^{c_t}_\star(s'|s,a)\right|\right)^2,
    \end{align*}
    \endgroup
    where $\hat{q}^t_h(s,a) := q_h(s,a|\pi^{c_t}_t,\widehat{P}^{c_t}_t)$ and $\hat{q}^{t,\star}_h(s,a):= q_h(s,a|\pi^{c_t}_\star, \widehat{P}^{c_t}_t)$ is the occupancy measure defined by an optimal context-dependent policy of the true CMDP $\pi_\star = (\pi^c_\star)_{c \in \C}$. In addition, $\Mhat_t (c) := (S,A,\widehat{P}^c_t, \hat{f}_t(c;\cdot,\cdot),s_0, H)$.
\end{lemma}

\begin{proof}
    Consider the following derivation.
    \begingroup\allowdisplaybreaks
    \begin{align*}
        &\sum_{t=1}^T 
        V^{\pi^{c_t}_\star}_{\M(c_t)}(s_0)
        -
        V^{\pi^{c_t}_\star}_{\Mhat_t(c_t)}(s_0)
        \\
        \tag{Value Difference,~\cref{lemma:val-diff-efroni}}
        = &
        \sum_{t=1}^T
        \mathop{\E}_{\pi^{c_t}_\star,\widehat{P}^{c_t}_t}
        \Bigg[ 
        \sum_{h=0}^{H-1} 
        \Bigg(\Big(\hat{f}_t(c_t,s_h,a_h) - f_\star(c_t,s_h,a_h)\Big) 
        \\
        & + \sum_{s' \in S}\left(\widehat{P}^{c_t}_t(s'|s^t_h,a^t_h) - P^{c_t}_\star(s'|s^t_h,a^t_h)\right)\cdot V^{\pi^{c_t}_\star}_{\M(c_t), {h+1}}(s') \Bigg)
        \Bigg|s_0\Bigg]
        \\
        = &
        \sum_{t=1}^T \sum_{h=0}^{H-1}\sum_{s \in S}\sum_{a \in A}\hat{q}^{t,\star}_h(s,a)\left( f_\star(c_t,s,a)-\hat{f}_t(c_t,s,a) \right)
        \\
        & +
        \sum_{t=1}^T \sum_{h=0}^{H-1}\sum_{s \in S}\sum_{a \in A}\hat{q}^{t,\star}_h(s,a) \sum_{s' \in S}\left( P^{c_t}_\star(s'|s,a)-\widehat{P}^{c_t}_t(s'|s,a) \right)V^{\pi^{c_t}_\star}_{\M(c_t),{h+1}}(s')
        \\
        \leq &
        \sum_{t=1}^T \sum_{h=0}^{H-1}\sum_{s \in S}\sum_{a \in A}\hat{q}^{t,\star}_h(s,a)\left( f_\star(c_t,s,a)-\hat{f}_t(c_t,s,a) \right)
        \\
        & +
       H \cdot \sum_{t=1}^T \sum_{h=0}^{H-1}\sum_{s \in S}\sum_{a \in A}\hat{q}^{t,\star}_h(s,a) \sum_{s' \in S}   \left| P^{c_t}_\star(s'|s,a)-\widehat{P}^{c_t}_t(s'|s,a) \right|
        \\
        = &
        \sum_{t=1}^T \sum_{h=0}^{H-1}\sum_{s \in S}\sum_{a \in A}\hat{q}^{t,\star}_h(s,a) \cdot \sqrt{\frac{\hat{\gamma} \cdot \hat{q}^t_h(s,a)}{\hat{\gamma} \cdot \hat{q}^t_h(s,a)}} \left(f_\star(c_t,s,a)-\hat{f}_t(c_t,s,a) \right)
        \\
        & +
        H \cdot \sum_{t=1}^T \sum_{h=0}^{H-1}\sum_{s \in S}\sum_{a \in A}\hat{q}^{t,\star}_h(s,a)\cdot \sqrt{\frac{\hat{\gamma} \cdot  \hat{q}^t_h(s,a)}{\hat{\gamma} \cdot \hat{q}^t_h(s,a)}} \sum_{s' \in S}  \left| P^{c_t}_\star(s'|s,a)-\widehat{P}^{c_t}_t(s'|s,a) \right|
        \\
        = &
        \sum_{t=1}^T \sum_{h=0}^{H-1}\sum_{s \in S}\sum_{a \in A}
        \sqrt{\frac{\hat{q}^{t,\star}_h(s,a)}{\hat{\gamma} \cdot \hat{q}^t_h(s,a)}} 
         \cdot 
         \sqrt{\hat{q}^{t,\star}_h(s,a) \cdot \hat{\gamma} \cdot \hat{q}^t_h(s,a)} \left( f_\star(c_t,s,a)-\hat{f}_t(c_t,s,a) \right)
        \\
        & +
        \sum_{t=1}^T \sum_{h=0}^{H-1}\sum_{s \in S}\sum_{a \in A} \sqrt{\frac{\hat{q}^{t,\star}_h(s,a)}{\hat{\gamma} \cdot \hat{q}^t_h(s,a)}} 
        \cdot 
        \sqrt{\hat{q}^{t,\star}_h(s,a) \cdot \hat{\gamma} \cdot \hat{q}^t_h(s,a)} \cdot H \cdot
        \left(\sum_{s' \in S}
        \left| P^{c_t}_\star(s'|s,a)-\widehat{P}^{c_t}_t(s'|s,a) \right|\right)
        \\
        \tag{Since for all $a,b \in \R$, $a b \leq \frac{1}{2}(a^2 +b^2)$ by AM-GM}
        \leq &
        \frac{1}{2}\sum_{t=1}^T \sum_{h=0}^{H-1}\sum_{s \in S}\sum_{a \in A}
        \left(\frac{\hat{q}^{t,\star}_h(s,a)}{\hat{\gamma} \cdot \hat{q}^t_h(s,a)} 
        +
        \hat{q}^{t,\star}_h(s,a) \cdot \hat{\gamma} \cdot \hat{q}^t_h(s,a) \cdot \left(\hat{f}_t(c_t,s,a) - f_\star(c_t,s,a)\right)^2\right)
        \\
        & +
        \frac{1}{2}\sum_{t=1}^T \sum_{h=0}^{H-1}\sum_{s \in S}\sum_{a \in A} \left( \frac{\hat{q}^{t,\star}_h(s,a)}{\hat{\gamma} \cdot \hat{q}^t_h(s,a)} + 
        \hat{q}^{t,\star}_h(s,a) \cdot \hat{\gamma} \cdot \hat{q}^t_h(s,a) \cdot
        H^2 \cdot \left(\sum_{s' \in S}
        \left|P^{c_t}_\star(s'|s,a)-\widehat{P}^{c_t}_t(s'|s,a) \right|\right)^2\right)
         \\
         \leq &
         \sum_{t=1}^T \sum_{h=0}^{H-1}\sum_{s \in S}\sum_{a \in A}
        \frac{\hat{q}^{t,\star}_h(s,a)}{\hat{\gamma} \cdot \hat{q}^t_h(s,a)} 
        \\
        & +
        \frac{\hat{\gamma}}{2}
        \sum_{t=1}^T \sum_{h=0}^{H-1}\sum_{s \in S}\sum_{a \in A}
        \hat{q}^t_h(s,a) \cdot \left(\hat{f}_t(c_t,s,a) - f_\star(c_t,s,a)\right)^2
        \\
        & +
        \frac{\hat{\gamma} \cdot H^2 }{2}
        \sum_{t=1}^T \sum_{h=0}^{H-1}\sum_{s \in S}\sum_{a \in A}
        \hat{q}^t_h(s,a) \cdot \left(\sum_{s' \in S} \left|\widehat{P}^{c_t}_t(s'|s,a) - P^{c_t}_\star(s'|s,a)\right|\right)^2,
    \end{align*}
    \endgroup
    as stated.
\end{proof}

\begin{corollary}[restatement of~\cref{lemma:val-diff-1-bound}]\label{corl:bound-of-term-1-UD}
    The following holds for any $\hat{\gamma}>0$. 
    \begingroup\allowdisplaybreaks
    \begin{align*}
        \sum_{t=1}^T 
        V^{\pi^{c_t}_\star}_{\M(c_t)}(s_0)
        -
        V^{\pi^{c_t}_\star}_{\Mhat_t(c_t)}(s_0)
        \leq & 
        \sum_{t=1}^T \sum_{h=0}^{H-1}\sum_{s \in S}\sum_{a \in A}
        \frac{\hat{q}^{t,\star}_h(s,a)}{\hat{\gamma} \cdot \hat{q}^t_h(s,a)} 
        \\
        & +
        2
        \hat{\gamma}\sum_{t=1}^T \mathop{\E}_{\pi^{c_t}_t, P^{c_t}_\star} \Bigg[\sum_{h=0}^{H-1}(\hat{f}_t(c_t,s_h,a_h) - f_\star(c_t,s_h,a_h))^2 \Big|s_0 \Bigg]
        \\
        & +
        29
        \hat{\gamma} H^4 \sum_{t=1}^T 
         \mathop{\E}_{\pi^{c_t}_t, P^{c_t}_\star} \Bigg[
        \sum_{h=0}^{H-1}
         D^2_H(P^{c_t}_\star(\cdot|s_{h},a_{h}), \widehat{P}^{c_t}_t(\cdot|s_{h},a_{h}))
        \Bigg| s_0\Bigg]
        .
    \end{align*}
    \endgroup
\end{corollary}

\begin{proof}
    Recall that $\hat{q}^t_h(s,a) := q_h(s,a|\pi^{c_t}_t,\widehat{P}^{c_t}_t)$ and $\hat{q}^{t,\star}_h(s,a):= q_h(s,a|\pi^{c_t}_\star, \widehat{P}^{c_t}_t)$.
    Consider the following derivation.
    \begingroup\allowdisplaybreaks
    \begin{align*}
        &\sum_{t=1}^T 
        V^{\pi^{c_t}_\star}_{\M(c_t)}(s_0)
        -
        V^{\pi^{c_t}_\star}_{\Mhat_t(c_t)}(s_0)
        \\
        \tag{By~\cref{lemma:bound-of-term-1-UD}}
        \leq & 
         \sum_{t=1}^T \sum_{h=0}^{H-1}\sum_{s \in S}\sum_{a \in A}
        \frac{\hat{q}^{t,\star}_h(s,a)}{\hat{\gamma} \cdot \hat{q}^t_h(s,a)} 
        \\
        & +
        \frac{\hat{\gamma}}{2}
        \sum_{t=1}^T \sum_{h=0}^{H-1}\sum_{s \in S}\sum_{a \in A}
        \hat{q}^t_h(s,a) \cdot \left(\hat{f}_t(c_t,s,a) - f_\star(c_t,s,a)\right)^2
        \\
        & +
        \frac{\hat{\gamma} \cdot H^2 }{2}
        \sum_{t=1}^T \sum_{h=0}^{H-1}\sum_{s \in S}\sum_{a \in A}
        \hat{q}^t_h(s,a) \cdot \left(\sum_{s' \in S} \left|\widehat{P}^{c_t}_t(s'|s,a) - P^{c_t}_\star(s'|s,a)\right|\right)^2
        \\
        \leq &
        \sum_{t=1}^T \sum_{h=0}^{H-1}\sum_{s \in S}\sum_{a \in A}
        \frac{\hat{q}^{t,\star}_h(s,a)}{\hat{\gamma} \cdot \hat{q}^t_h(s,a)} 
        \\
        \tag{By~\cref{lemma:rewards-true-occ-mesure}}
        & +
        \frac{3}{2}
        \hat{\gamma}\sum_{t=1}^T \mathop{\E}_{\pi^{c_t}_t, P^{c_t}_\star} \Bigg[\sum_{h=0}^{H-1}(\hat{f}_t(c_t,s_h,a_h) - f_\star(c_t,s_h,a_h))^2 \Big|s_0 \Bigg]
        \\
        & +
        \frac{9}{2} \hat{\gamma} H^2  \sum_{t=1}^{T}  
        \mathop{\E}_{\pi^{c_t}_t, P^{c_t}_\star} \Bigg[
        \sum_{h=0}^{H-1}
         D^2_H(P^{c_t}_\star(\cdot|s_{h},a_{h}), \widehat{P}^{c_t}_t(\cdot|s_{h},a_{h}))
        \Bigg| s_0\Bigg]
        \\
        \tag{By~\cref{lemma:dynamics-true-occ-mesure}}
        & +
        \frac{48}{2}
        \hat{\gamma} H^4 \sum_{t=1}^T 
         \mathop{\E}_{\pi^{c_t}_t, P^{c_t}_\star} \Bigg[
        \sum_{h=0}^{H-1}
         D^2_H(P^{c_t}_\star(\cdot|s_{h},a_{h}), \widehat{P}^{c_t}_t(\cdot|s_{h},a_{h}))
        \Bigg| s_0\Bigg]
        \\
        \leq &
        \sum_{t=1}^T \sum_{h=0}^{H-1}\sum_{s \in S}\sum_{a \in A}
        \frac{\hat{q}^{t,\star}_h(s,a)}{\hat{\gamma} \cdot \hat{q}^t_h(s,a)} 
        \\
        & +
        2
        \hat{\gamma}\sum_{t=1}^T \mathop{\E}_{\pi^{c_t}_t, P^{c_t}_\star} \Bigg[\sum_{h=0}^{H-1}(\hat{f}_t(c_t,s_h,a_h) - f_\star(c_t,s_h,a_h))^2 \Big|s_0 \Bigg]
        \\
        % & +
        % \frac{9}{2} \gamma H^2  \sum_{t=1}^{T}  
        % \mathop{\E}_{\pi^{c_t}_t, P^{c_t}_\star} \Bigg[
        % \sum_{h=0}^{H-1}
        %  D^2_H(P^{c_t}_\star(\cdot|s_{h},a_{h}), \widehat{P}^{c_t}_t(\cdot|s_{h},a_{h}))
        % \Bigg| s_0\Bigg]
        % \\
        % \tag{By~\cref{lemma:dynamics-true-occ-mesure}}
        & +
        29
        \hat{\gamma} H^4 \sum_{t=1}^T 
         \mathop{\E}_{\pi^{c_t}_t, P^{c_t}_\star} \Bigg[
        \sum_{h=0}^{H-1}
         D^2_H(P^{c_t}_\star(\cdot|s_{h},a_{h}), \widehat{P}^{c_t}_t(\cdot|s_{h},a_{h}))
        \Bigg| s_0\Bigg]
        .
    \end{align*}
    \endgroup
\end{proof}

\begin{lemma}[restatement of~\cref{lemma:comulative-bound-term-2-UD} ]\label{lemma:bound-of-term-2-UD}
    For every round $t \in [T]$ and a context $c_t \in \C$, the optimal solution $\hat{q}^t$ for the maximization problem in~\cref{eq:max-problem-unknown-dynamics} satisfies the following,
    \begin{align*}
        V^{\pi^{c_t}_\star}_{\Mhat_t(c_t)}(s_0)
        -
        V^{\pi^{c_t}_t}_{\Mhat_t(c_t)}(s_0)
        % = &
        % \sum_{h=0}^{H-1}\sum_{s \in S}\sum_{a \in A} \hat{q}^{t,\star}_h(s,a) \cdot  \hat{f}_t(c_t,s,a) -\sum_{h=0}^{H-1}\sum_{s \in S}\sum_{a \in A}\hat{q}^t_h(s,a) \cdot \hat{f}_t(c_t,s,a)
        %  \\
         \leq 
         \frac{H|S||A|}{\gamma}- \sum_{h=0}^{H-1}\sum_{s \in S}\sum_{a \in A} \frac{\hat{q}^{t,\star}_h(s,a)}{\gamma \cdot \hat{q}^t_h(s,a)} ,
    \end{align*}
    where $\hat{q}^t_h(s,a) := q_h(s,a|\pi^{c_t}_t,\widehat{P}^{c_t}_t)$ and $\hat{q}^{t,\star}_h(s,a):= q_h(s,a|\pi^{c_t}_\star, \widehat{P}^{c_t}_t)$ is the occupancy measure defined by an optimal context-dependent policy of the true CMDP $\pi_\star = (\pi^c_\star)_{c \in \C}$, recalling $\Mhat_t (c) := (S,A,\widehat{P}^c_t, \hat{f}_t(c;\cdot,\cdot),s_0, H)$.
\end{lemma}

\begin{proof}
    For every round $t \in [T]$ let $\hat{L}_t(q;c_t)$ denote the objective of the maximization problem in~\cref{eq:max-problem-unknown-dynamics} in round $t$, i.e.,
    \begin{align*}
        \hat{L}_t(q;c_t) = \sum_{h=0}^{H-1}\sum_{s \in S}\sum_{a \in A}q_h(s,a)\cdot \hat{f}_t(c_t,s,a) + \frac{1}{\gamma} \sum_{h=0}^{H-1}\sum_{s \in S}\sum_{a \in A} \log(q_h(s,a)).
    \end{align*}
    Thus the first-order derivation is
    \begin{align*}
         \hat{L}_t^\prime(q;c_t) = \sum_{h=0}^{H-1}\sum_{s \in S}\sum_{a \in A}\left(\hat{f}_t(c_t,s,a) + \frac{1}{\gamma \cdot q_h(s,a)} \right).
    \end{align*}
    
    Let $\pi_\star = (\pi^c_\star)_{c \in \C}$ denote an optimal context-dependent policy for the true CMDP. For every round $t$, the occupancy measures $\hat{q}^{t,\star}_h(s,a):= q_h(s,a|\pi^{c_t}_\star, \widehat{P}^{c_t}_t)$ is a feasible solution (since $\hat{q}^{t,\star} \in \mu(\widehat{P}^{c_t}_t)$). Since $\hat{q}^t$ is the optimal solution, the following holds by first order optimality conditions. 
    \begin{align*}
         &\sum_{h=0}^{H-1}\sum_{s \in S}\sum_{a \in A}\left(\hat{f}_t(c_t,s,a) + \frac{1}{\gamma \cdot \hat{q}^t_h(s,a)} \right)\left(\hat{q}^{t,\star}_h(s,a) -\hat{q}^t_h(s,a)\right) \leq 0
         \\
         \implies &
         \sum_{h=0}^{H-1}\sum_{s \in S}\sum_{a \in A} \hat{q}^{t,\star}_h(s,a) \cdot \left( \hat{f}_t(c_t,s,a) + \frac{1}{\gamma \cdot \hat{q}^t_h(s,a)}\right) -\sum_{h=0}^{H-1}\sum_{s \in S}\sum_{a \in A}\hat{q}^t_h(s,a) \cdot \hat{f}_t(c_t,s,a) -\frac{H|S||A|}{\gamma} \leq 0
         \\
         \implies & 
         \sum_{h=0}^{H-1}\sum_{s \in S}\sum_{a \in A} \hat{q}^{t,\star}_h(s,a) \cdot  \hat{f}_t(c_t,s,a)
         -\sum_{h=0}^{H-1}\sum_{s \in S}\sum_{a \in A}\hat{q}^t_h(s,a) \cdot \hat{f}_t(c_t,s,a) \leq \frac{H|S||A|}{\gamma} - \sum_{h=0}^{H-1}\sum_{s \in S}\sum_{a \in A} \frac{\hat{q}^{t,\star}_h(s,a)}{\gamma \cdot \hat{q}^t_h(s,a)} .
    \end{align*}
    By definition we have for every round $t$,
    \[
        V^{\pi^{c_t}_\star}_{\Mhat_t(c_t)}(s_0)
        -
        V^{\pi^{c_t}_t}_{\Mhat_t(c_t)}(s_0)
        = 
        \sum_{h=0}^{H-1}\sum_{s \in S}\sum_{a \in A} \hat{q}^{t,\star}_h(s,a) \cdot  \hat{f}_t(c_t,s,a) -\sum_{h=0}^{H-1}\sum_{s \in S}\sum_{a \in A}\hat{q}^t_h(s,a) \cdot \hat{f}_t(c_t,s,a),
    \]
    hence the lemma follows.
    
\end{proof}

\begin{lemma}\label{lemma:term-3-UD}
    The following holds.
    \begin{align*}
        \sum_{t=1}^T
        V^{\pi^{c_t}_t}_{\Mhat_t(c_t)}(s_0)
        - 
        V^{\pi^{c_t}_t}_{\M(c_t)}(s_0)
        \leq &
        \sum_{t=1}^T \sum_{h=0}^{H-1}\sum_{s \in S}\sum_{a \in A}
        q^t_h(s,a)\cdot(\hat{f}_t(c_t,s,a) - f_\star(c_t,s,a))
        \\
        & +
        H \sum_{t=1}^T \sum_{h=0}^{H-1}\sum_{s \in S}\sum_{a \in A}
        q^t_h(s,a)\cdot \sum_{s' \in S}\left|\widehat{P}^{c_t}_t(s'|s,a) - P^{c_t}_\star(s'|s,a)\right|,
    \end{align*}
    where  $q^t_h(s,a):= q_h(s,a|\pi^{c_t}_t, P^{c_t}_\star)$.
\end{lemma}

%\Orin{Maybe need to fix that lemma}

\begin{proof}
    By the Value Difference Lemma,~(\cref{lemma:val-diff-efroni}), the following holds.
    \begingroup\allowdisplaybreaks
    \begin{align*}
        &\sum_{t=1}^T
        V^{\pi^{c_t}_t}_{\Mhat_t(c_t)}(s_0)
        - 
        V^{\pi^{c_t}_t}_{\M(c_t)}(s_0)
        \\
        = &
        \sum_{t=1}^T  \mathop{\E}_{\pi^{c_t}_t,P^{c_t}_\star}
        \left[ 
        \sum_{h=0}^{H-1} \left(\hat{f}_t(c_t,s_h,a_h) - f_\star(c_t,s_h,a_h)\right) + \sum_{s' \in S}\left(\widehat{P}^{c_t}_t(s'|s^t_h,a^t_h) - P^{c_t}_\star(s'|s^t_h,a^t_h)\right)\cdot V^{\pi^{c_t}_t}_{\Mhat_t(c_t), {h+1}}(s')
        \Big| s_0\right]
        \\
        = &
        \sum_{t=1}^T \sum_{h=0}^{H-1}\sum_{s \in S}\sum_{a \in A}
        q^t_h(s,a)\cdot(\hat{f}_t(c_t,s,a) - f_\star(c_t,s,a))
        \\
        & +
        \sum_{t=1}^T \sum_{h=0}^{H-1}\sum_{s \in S}\sum_{a \in A}
        q^t_h(s,a)\cdot \sum_{s' \in S}(\widehat{P}^{c_t}_t(s'|s,a) - P^{c_t}_\star(s'|s,a)) \cdot  V^{\pi^{c_t}_t}_{\Mhat_t(c_t),{h+1}}(s')
        \\
        \leq &
        \sum_{t=1}^T \sum_{h=0}^{H-1}\sum_{s \in S}\sum_{a \in A}
        q^t_h(s,a)\cdot(\hat{f}_t(c_t,s,a) - f_\star(c_t,s,a))
        \\
        & +
        H \sum_{t=1}^T \sum_{h=0}^{H-1}\sum_{s \in S}\sum_{a \in A}
        q^t_h(s,a)\cdot \sum_{s' \in S}\left|\widehat{P}^{c_t}_t(s'|s,a) - P^{c_t}_\star(s'|s,a)\right|
        ,
    \end{align*}
    \endgroup
    as stated.
\end{proof}

\begin{lemma}\label{lemma:bound-of-term-3-UD}
    The following holds for any parameter $p_1 > 0$.
    \begin{align*}
        \sum_{t=1}^T \sum_{h=0}^{H-1}\sum_{s \in S}\sum_{a \in A}q^t_h(s,a)(\hat{f}_t(c_t,s,a)  - &f_\star(c_t,s,a))
        \leq 
        \frac{TH}{2 p_1}    
        \\
        & +
        \frac{p_1}{2}\sum_{t=1}^T \mathop{\E}_{\pi^{c_t}_t, P^{c_t}_\star} \Bigg[ \sum_{h=0}^{H-1}
        \left(\hat{f}_t(c_t,s_h,a_h) - f_\star(c_t,s_h,a_h)\right)^2
        \Bigg|s_0 \Bigg]
        .
    \end{align*}
    %where $p_1 > 0$ is a constant that will be determined later.
\end{lemma}

\begin{proof}
    Consider the following derivation, where $q^t_h(s,a):= q_h(s,a|\pi^{c_t}_t, P^{c_t}_\star)$.
    \begingroup\allowdisplaybreaks
    \begin{align*}
        &\sum_{t=1}^T \sum_{h=0}^{H-1}\sum_{s \in S}\sum_{a \in A}q^t_h(s,a)\left(\hat{f}_t(c_t,s,a) - f_\star(c_t,s,a)\right)
        \\
        = &
        \sum_{t=1}^T \sum_{h=0}^{H-1}\sum_{s \in S}\sum_{a \in A}\sqrt{\frac{q^t_h(s,a)}{p_1}}\cdot \sqrt{p_1 \cdot q^t_h(s,a)}\left(\hat{f}_t(c_t,s,a) - f_\star(c_t,s,a)\right)
        \\
        \tag{Since for all $a,b \in \R$, $a  b \leq \frac{1}{2}(a^2 +b^2)$ by AM-GM}
        \leq &
        \frac{1}{2}\sum_{t=1}^T \sum_{h=0}^{H-1}\sum_{s \in S}\sum_{a \in A} \left(\frac{q^t_h(s,a)}{p_1} + p_1 \cdot q^t_h(s,a)\left(\hat{f}_t(c_t,s,a) - f_\star(c_t,s,a)\right)^2 \right)
        \\
        = &
        \frac{1}{2 p_1}\sum_{t=1}^T \underbrace{\sum_{h=0}^{H-1}\sum_{s \in S}\sum_{a \in A}q^t_h(s,a)}_{\leq H}
        +
        \frac{p_1}{2}\sum_{t=1}^T \sum_{h=0}^{H-1}\sum_{s \in S}\sum_{a \in A}
        q^t_h(s,a)\left(\hat{f}_t(c_t,s,a) - f_\star(c_t,s,a)\right)^2
        \\
        \leq &
        \frac{TH}{2 p_1}    
        +
        \frac{p_1}{2}\sum_{t=1}^T \sum_{h=0}^{H-1}\sum_{s \in S}\sum_{a \in A}
        q^t_h(s,a)\left(\hat{f}_t(c_t,s,a) - f_\star(c_t,s,a)\right)^2
        \\
        = &
        \frac{TH}{2 p_1}    
        +
        \frac{p_1}{2}\sum_{t=1}^T \mathop{\E}_{\pi^{c_t}_t, P^{c_t}_\star} \Bigg[ \sum_{h=0}^{H-1}
        \left(\hat{f}_t(c_t,s_h,a_h) - f_\star(c_t,s_h,a_h)\right)^2
        \Bigg|s_0 \Bigg]
        .
    \end{align*}
    \endgroup
\end{proof}

\begin{lemma}\label{lemma:bound-of-term-3-UD-dynamics}
    The following holds for any parameter $p_2 > 0$.
    \begin{align*}
        \sum_{t=1}^T \sum_{h=0}^{H-1}\sum_{s \in S}\sum_{a \in A}q^t_h(s,a) \sum_{s' \in S}|\widehat{P}^{c_t}_t(s'|s,a) - & P^{c_t}_\star(s'|s,a)|
        \leq 
        \frac{TH}{2 p_2}
        \\
        & +
        2{p_2}\sum_{t=1}^T 
        \mathop{\E}_{\pi^{c_t}_t, P^{c_t}_\star} \Bigg[
        \sum_{h=0}^{H-1}
         D^2_H(P^{c_t}_\star(\cdot|s_{h},a_{h}), \widehat{P}^{c_t}_t(\cdot|s_{h},a_{h}))
        \Bigg| s_0\Bigg]
        .
    \end{align*}
\end{lemma}

\begin{proof}
    Recall $q^t_h(s,a):= q_h(s,a|\pi^{c_t}_t, P^{c_t}_\star)$ and consider the following derivation.
    \begingroup\allowdisplaybreaks
    \begin{align*}
        &\sum_{t=1}^T \sum_{h=0}^{H-1}\sum_{s \in S}\sum_{a \in A}q^t_h(s,a) \sum_{s' \in S}\left|\widehat{P}^{c_t}_t(s'|s,a) - P^{c_t}_\star(s'|s,a)\right|
        \\
        = &
        \sum_{t=1}^T \sum_{h=0}^{H-1}\sum_{s \in S}\sum_{a \in A} 
        \sqrt{\frac{q^t_h(s,a)}{p_2}} \cdot \sqrt{p_2 \cdot q^t_h(s,a) }
        \left(\sum_{s' \in S}        \left|\widehat{P}^{c_t}_t(s'|s,a) - P^{c_t}_\star(s'|s,a)\right|\right)
        \\
        \tag{Since for all $a,b \in \R$, $a  b \leq \frac{1}{2}(a^2 +b^2)$ by AM-GM}
        \leq &
        \frac{1}{2}\sum_{t=1}^T \sum_{h=0}^{H-1}\sum_{s \in S}\sum_{a \in A} \left(\frac{q^t_h(s,a)}{p_2} + p_2 \cdot q^t_h(s,a) 
        \left(\sum_{s'\in S}\left|\widehat{P}^{c_t}_t(s'|s,a) - P^{c_t}_\star(s'|s,a)\right|\right)^2\right)
        \\
        = &
        \frac{1}{2 p_2}\sum_{t=1}^T \underbrace{\sum_{h=0}^{H-1}\sum_{s \in S}\sum_{a \in A} q^t_h(s,a) }_{\leq H}
        % \\
        % & 
        +
        \frac{p_2}{2}\sum_{t=1}^T \sum_{h=0}^{H-1}\sum_{s \in S}\sum_{a \in A}
        q^t_h(s,a) \left(\sum_{s' \in S}\left|\widehat{P}^{c_t}_t(s'|s,a) - P^{c_t}_\star(s'|s,a)\right|\right)^2
        \\
        \leq &
        \frac{TH}{2 p_2}
        +
        \frac{p_2}{2}\sum_{t=1}^T \sum_{h=0}^{H-1}\sum_{s \in S}\sum_{a \in A}
        q^t_h(s,a) \left(\sum_{s' \in S}\left|\widehat{P}^{c_t}_t(s'|s,a) - P^{c_t}_\star(s'|s,a)\right|\right)^2
        \\
        \tag{$\|\cdot\|^2_1 \leq 4 D^2_H$}
        \leq &
        \frac{TH}{2 p_2}
        +
        2{p_2}\sum_{t=1}^T \sum_{h=0}^{H-1}\sum_{s \in S}\sum_{a \in A}
        q^t_h(s,a) \cdot D^2_H(\widehat{P}^{c_t}_t(\cdot|s,a), P^{c_t}_\star(\cdot|s,a))
        \\
        = &
        \frac{TH}{2 p_2}
        +
        2{p_2}\sum_{t=1}^T 
        \mathop{\E}_{\pi^{c_t}_t, P^{c_t}_\star} \Bigg[
        \sum_{h=0}^{H-1}
         D^2_H(P^{c_t}_\star(\cdot|s_{h},a_{h}), \widehat{P}^{c_t}_t(\cdot|s_{h},a_{h}))
        \Bigg| s_0\Bigg]
        .
    \end{align*}
    \endgroup
\end{proof}

\begin{corollary}[restatement of~\cref{lemma:term-3-UD-main}]\label{corl:term-3-UD}
    The following holds for any two parameters $p_1, p_2 >0$.
    \begin{align*}
        \sum_{t=1}^T
        V^{\pi^{c_t}_t}_{\Mhat_t(c_t)}(s_0)
        - 
        V^{\pi^{c_t}_t}_{\M(c_t)}(s_0)
        \leq 
        \frac{TH}{2 p_1}    
        & +
        \frac{p_1}{2}\sum_{t=1}^T \mathop{\E}_{\pi^{c_t}_t, P^{c_t}_\star} \Bigg[ \sum_{h=0}^{H-1}
        \left(\hat{f}_t(c_t,s_h,a_h) - f_\star(c_t,s_h,a_h)\right)^2
        \Bigg|s_0 \Bigg]
        \\
        & +
        \frac{TH}{2 p_2}
        +
        2{p_2}\sum_{t=1}^T 
        \mathop{\E}_{\pi^{c_t}_t, P^{c_t}_\star} \Bigg[
        \sum_{h=0}^{H-1}
         D^2_H(P^{c_t}_\star(\cdot|s_{h},a_{h}), \widehat{P}^{c_t}_t(\cdot|s_{h},a_{h}))
        \Bigg| s_0\Bigg]
        .
    \end{align*}
\end{corollary}

\begin{proof}
    Consider the following derivation.
    \begingroup\allowdisplaybreaks
    \begin{align*}
        &\sum_{t=1}^T
        V^{\pi^{c_t}_t}_{\Mhat_t(c_t)}(s_0)
        - 
        V^{\pi^{c_t}_t}_{\M(c_t)}(s_0)
        \\
        \tag{By~\cref{lemma:term-3-UD}}
        \leq &
        \sum_{t=1}^T \sum_{h=0}^{H-1}\sum_{s \in S}\sum_{a \in A}
        q^t_h(s,a)\cdot(\hat{f}_t(c_t,s,a) - f_\star(c_t,s,a))
        \\
        & +
        H \sum_{t=1}^T \sum_{h=0}^{H-1}\sum_{s \in S}\sum_{a \in A}
        q^t_h(s,a)\cdot \sum_{s' \in S}\left|\widehat{P}^{c_t}_t(s'|s,a) - P^{c_t}_\star(s'|s,a)\right|
        \\
        \tag{By~\cref{lemma:bound-of-term-3-UD}}
        \leq &
        \frac{TH}{2 p_1}    
        +
        \frac{p_1}{2}\sum_{t=1}^T \mathop{\E}_{\pi^{c_t}_t, P^{c_t}_\star} \Bigg[ \sum_{h=0}^{H-1}
        \left(\hat{f}_t(c_t,s_h,a_h) - f_\star(c_t,s_h,a_h)\right)^2
        \Bigg|s_0 \Bigg]
        \\
        \tag{By~\cref{lemma:bound-of-term-3-UD-dynamics}}
        & +
        \frac{TH}{2 p_2}
        +
        2{p_2}\sum_{t=1}^T 
        \mathop{\E}_{\pi^{c_t}_t, P^{c_t}_\star} \Bigg[
        \sum_{h=0}^{H-1}
         D^2_H(P^{c_t}_\star(\cdot|s_{h},a_{h}), \widehat{P}^{c_t}_t(\cdot|s_{h},a_{h}))
        \Bigg| s_0\Bigg]
        .
    \end{align*}
    \endgroup
\end{proof}

%\subsubsection{Transition to the oracles regret}
% \begin{lemma}[transition to oracle regret for rewards]\label{lemma:transition-to-oracles-regret-rewards-UD}
%     For any $\delta \in (0,1)$, with probability at least $1-\delta$ the following holds.
%     \begin{align*}
%         \sum_{t=1}^T \sum_{h=0}^{H-1}\sum_{s \in S}\sum_{a \in A} q^t_h(s,a) \cdot \left(\hat{f}_t(c_t,s,a) - f_\star(c_t,s,a)\right)^2
%         \leq
%         2\cdot
%         \Regrv_{TH}(\OLSR^\F)
%         + 16 H\log(1/\delta),
%     \end{align*}
%     where  $q^t_h(s,a):= q_h(s,a|\pi^{c_t}_t, P^{c_t}_\star)$.
% \end{lemma}

% %\Orin{Maybe need to fix that lemma, and the oracle definition}

% \begin{proof}
%     Identical to the proofs of~\cref{lemma:transition-to-oracles-regret-rewards}.
% \end{proof}

\subsubsection{Regret Bound}\label{Appendix-subsec:regret}

The following theorem states our main result, which is a regret bound for~\cref{alg:OMG-CMDP}.

\begin{theorem*}[restatement of~\cref{thm:regret-main}]\label{thm:regret}
    For any $\delta \in (0,1)$, let $\gamma = \sqrt{\frac{|S||A|T}{31H^3 \left( 2\Regrv_{TH}(\OLSR^\F) + \Regrv_{TH}(\OLLR^{\Fp}) + 18H \log(2H/\delta) \right) }}$.
    Then, with probability at least $1-\delta$ it holds that
    \begin{align*}
        \Regrv_T(\text{OMG-CMDP!}) \leq 
        \widetilde{O}\left(H^{2.5} \sqrt{ T|S||A| 
        \left( 
        \Regrv_{TH}(\OLSR^\F) + \Regrv_{TH}(\OLLR^{\Fp}) + H \log\delta^{-1} \right)}\right)
        .
    \end{align*}
\end{theorem*}

\begin{proof}
    %For all $t$, let $q^t_h(s,a):= q_h(s,a|\pi^{c_t}_t, P^{c_t}_\star)$.
    By~\cref{lemma:lsro-regret}, with probability at least $1-\delta/2$, it holds that
    \begin{align*}
        \sum_{t=1}^T \mathop{\E}_{\pi^{c_t}_t, P^{c_t}_\star} \Bigg[\sum_{h=0}^{H-1}\left(\hat{f}_t(c_t,s_h,a_h) - f_\star(c_t,s_h,a_h)\right)^2 \Bigg|s_0\Bigg]
        \leq
        2
        \Regrv_{TH}(\OLSR^\F)
        + 16 H \log(2/\delta).
    \end{align*}
    By~\cref{lemma:llrs-regret}, with probability at least $1-\delta/2$, it holds that
    \begin{align*}
        \sum_{t=1}^T 
        \mathop{\E}_{\pi^{c_t}_t, P^{c_t}_\star} \Bigg[
        \sum_{h=0}^{H-1}D^2_H(P^{c_t}_\star(\cdot|s_{h},a_{h}),\widehat{P}^{c_t}_t(\cdot|s_{h},a_{h}))
        \Bigg| s_0\Bigg]
        \leq 
        \Regrv_{TH}(\OLLR^{\Fp}) + 2H \log(2H/\delta).
    \end{align*}
    We prove a regret bound under those two good events.
    \begingroup\allowdisplaybreaks
    \begin{align*}
        &\Regrv_T(\text{OMG-CMDP!})
        \\
        = &
        \sum_{t=1}^T V^{\pi^{c_t}_\star}_{\M(c_t)}(s_0)
        - V^{\pi^{c_t}_t}_{\M(c_t)}(s_0)
        \\
        = &
        \sum_{t=1}^T 
        V^{\pi^{c_t}_\star}_{\M(c_t)}(s_0)
        -
        V^{\pi^{c_t}_\star}_{\Mhat_t(c_t)}(s_0)
        +
        \sum_{t=1}^T 
        V^{\pi^{c_t}_\star}_{\Mhat_t(c_t)}(s_0)
        -
        V^{\pi^{c_t}_t}_{\Mhat_t(c_t)}(s_0)
        +
        \sum_{t=1}^T 
        V^{\pi^{c_t}_t}_{\Mhat_t(c_t)}(s_0)
        - 
        V^{\pi^{c_t}_t}_{\M(c_t)}(s_0)
        \\
        \tag{By~\cref{corl:bound-of-term-1-UD}, for $\hat{\gamma} = \gamma.$}
        \leq &
        \sum_{t=1}^T \sum_{h=0}^{H-1}\sum_{s \in S}\sum_{a \in A}
        \frac{\hat{q}^{t,\star}_h(s,a)}{\gamma \cdot \hat{q}^t_h(s,a)} 
        \\
        & +
        2
        \gamma\sum_{t=1}^T \mathop{\E}_{\pi^{c_t}_t, P^{c_t}_\star} \Bigg[\sum_{h=0}^{H-1}(\hat{f}_t(c_t,s_h,a_h) - f_\star(c_t,s_h,a_h))^2 \Big|s_0 \Bigg]
        \\
        & +
        29
        \gamma H^4 \sum_{t=1}^T 
         \mathop{\E}_{\pi^{c_t}_t, P^{c_t}_\star} \Bigg[
        \sum_{h=0}^{H-1}
         D^2_H(P^{c_t}_\star(\cdot|s_{h},a_{h}), \widehat{P}^{c_t}_t(\cdot|s_{h},a_{h}))
        \Bigg| s_0\Bigg]
        \\
        \tag{By~\cref{lemma:bound-of-term-2-UD}, applied for each $t \in [T]$}
        & +
        \frac{H|S||A|T}{\gamma}- \sum_{t=1}^T\sum_{h=0}^{H-1}\sum_{s \in S}\sum_{a \in A} \frac{\hat{q}^{t,\star}_h(s,a)}{\gamma \cdot \hat{q}^t_h(s,a)}
        \\
        \tag{By~\cref{corl:term-3-UD}}
        & + 
        \frac{TH}{2 p_1}
        \\
        & +
        \frac{p_1}{2}\sum_{t=1}^T \mathop{\E}_{\pi^{c_t}_t, P^{c_t}_\star} \Bigg[ \sum_{h=0}^{H-1}
        \left(\hat{f}_t(c_t,s_h,a_h) - f_\star(c_t,s_h,a_h)\right)^2
        \Bigg|s_0 \Bigg]
        \\
        & +
        \frac{TH}{2 p_2}
        +
        2{p_2}\sum_{t=1}^T 
        \mathop{\E}_{\pi^{c_t}_t, P^{c_t}_\star} \Bigg[
        \sum_{h=0}^{H-1}
         D^2_H(P^{c_t}_\star(\cdot|s_{h},a_{h}), \widehat{P}^{c_t}_t(\cdot|s_{h},a_{h}))
        \Bigg| s_0\Bigg]
        \\
        \tag{By the good events}
        \leq &
        2
        \gamma
        \left( 2\cdot
        \Regrv_{TH}(\OLSR^\F)
        + 16 H \log(2/\delta)\right)
        \\
        & +
        29\gamma H^4 
        \left( \Regrv_{TH}(\OLLR^{\Fp}) + 2H \log(2H/\delta) \right)
        \\
        & +
        \frac{H|S||A|T}{\gamma}
        \\
        & + 
        \frac{TH}{2 p_1}
         +
        \frac{p_1}{2}
        \left( 2\cdot
        \Regrv_{TH}(\OLSR^\F)
        + 16 H \log(2/\delta) \right)
        \\
        & +
        \frac{TH}{2 p_2}
        +
        2{p_2}
        \left( \Regrv_{TH}(\OLLR^{\Fp}) + 2H \log(2H/\delta)\right)
        \\
        \leq &
        \gamma \cdot 31 H^4 
        \left( 2\cdot
        \Regrv_{TH}(\OLSR^\F) + \Regrv_{TH}(\OLLR^{\Fp}) + 18H \log(2H/\delta) \right)
        +
        \frac{H|S||A|T}{\gamma}
        \\
        & + 
        \frac{TH}{2 p_1}
         +
        \frac{p_1}{2}
        \left( 2\cdot
        \Regrv_{TH}(\OLSR^\F)
        + 16 H \log(2/\delta) \right)
        \\
        & +
        \frac{TH}{2 p_2}
        +
        2{p_2}
        \left( \Regrv_{TH}(\OLLR^{\Fp}) + 2H \log(2H/\delta)\right)
        \\
        \tag{For $\gamma = \sqrt{\frac{|S||A|T}{31H^3 \left( 2\cdot
        \Regrv_{TH}(\OLSR^\F) + \Regrv_{TH}(\OLLR^{\Fp}) + 18H \log(2H/\delta) \right) }}$}
        = &
        2H^{2.5} \sqrt{ 31 T|S||A| 
        \left( 2\cdot
        \Regrv_{TH}(\OLSR^\F) + \Regrv_{TH}(\OLLR^{\Fp}) + 18H \log(2H/\delta) \right)}
        \\
        \tag{For $p_1 = \sqrt{\frac{TH}{2
        \Regrv_{TH}(\OLSR^\F)
        + 16 H \log(2/\delta)}}$}
        & + 
        \sqrt{TH
        \left( 2\cdot
        \Regrv_{TH}(\OLSR^\F)
        + 16 H \log(2/\delta) \right)}
        \\
        \tag{For $p_2 = \sqrt{\frac{TH}{4\left( \Regrv_{TH}(\OLLR^{\Fp}) + 2H \log(2H/\delta)\right)}}$}
        & +
        2
        \sqrt{TH\left( \Regrv_{TH}(\OLLR^{\Fp}) + 2H \log(2H/\delta)\right)}
        \\
        = &
        \widetilde{O}\left(H^{2.5} \sqrt{ T|S||A| 
        \left( 
        \Regrv_{TH}(\OLSR^\F) + \Regrv_{TH}(\OLLR^{\Fp}) + H \log\delta^{-1} \right)}\right)
        .
    \end{align*}
    \endgroup
    Since the good events hold with probability at least $1-\delta$, so is the regret bound above.
\end{proof}

\section{Approximated Solution}\label{Appendix-subsec:approx-sol}
The objective of optimization problem in \cref{eq:max-problem-unknown-dynamics} is defined as a sum of a self-concordant barrier function (the log function) and a linear function. 
Hence, the optimal solution for the problem can be approximated efficiently using interior-point convex optimization algorithms such as Newton's Method. 
These algorithms return an $\epsilon$-approximated solution and has running time of $O(poly(d)\log \epsilon^{-1})$, where $d$ is the dimension of the problem
%To obtain $\epsilon$-approximated solution we need to run an optimization algorithm for $O(\frac{\gamma}{\epsilon})$ iterations 
\citep{nesterov1994interior}.

Suppose that in each round $t$ we derive the policy $\pi^{c_t}_t$ using, instead of the optimal solution, an occupancy measure $\hat{q}^t$ that yields an $\epsilon$-approximation to the objective of the optimization problem in \cref{eq:max-problem-unknown-dynamics}.
The following analysis shows that for $\epsilon = \frac{1}{16 \gamma T}$, we obtain a similar regret guarantee. In addition, by our choice of $\gamma$, the running time complexity of the optimization algorithm is $\poly(|S|,|A|,H, \log(T))$.
% We start by bounding the difference between the optimal and the approximated iterates.

\subsection{Regret Analysis}

The following lemma bounds the difference between the optimal and the approximated iterates.
\begin{lemma}[iterates difference, restatement of~\cref{lemma:approx-iterates-main}]\label{lemma:approx-iterates}
    For every round $t \in [T]$ let 
    \textbf{\begin{align*}
        \hat{L}_t(q;c_t) = \sum_{h=0}^{H-1}\sum_{s \in S}\sum_{a \in A}q_h(s,a)\cdot \hat{f}_t(c_t,s,a) + \frac{1}{\gamma} \sum_{h=0}^{H-1}\sum_{s \in S}\sum_{a \in A} \log(q_h(s,a)).
    \end{align*}}
    denote the objective of the optimization problem in~\cref{eq:max-problem-unknown-dynamics}.
    Let $\widetilde{q} \in \arg\max_{q \in \mu(\widehat{P}^{c_t}_t)}\hat{L}_t(q;c_t)$.
    Let $q \in \mu(\widehat{P}^{c_t}_t)$ and suppose that $  \hat{L}_t(\widetilde{q};c_t) -\hat{L}_t(q;c_t)\leq \epsilon$. Then,
    \begin{align*}
    \sum_{s \in S} \sum_{a \in A} \sum_{h=0}^{H-1}
    \left( \frac{q_h(s,a)}{\widetilde{q}_h(s,a)} -1\right)^2
    \leq 4 \epsilon \gamma
    .
\end{align*}
\end{lemma}

\begin{proof}
Recall the Bregman divergence with respect to a $\theta$-self-concordant barrier $R$ as follows:
   \begin{align*}
       B_R(y||x) = R(y) - R(x) - \nabla R(x) \cdot (y-x).
   \end{align*}
\noindent 
We have the following lower bound on the Bregman divergence, where $\norm{x-y}_x^2 = (x-y)^\top \nabla^2 R(x) (x-y)$.
\begin{align*}
    B_R(y||x) \geq \rho\left( \| y-x \|_x \right) 
    \text{  for  }
    \rho(z) = z-\log(1+z).
\end{align*}
We have that $\rho(z) \geq z^2/4$ for all $z \in [0,1]$, and $\| y-x \|_x  \leq 1$.
Thus,
\begin{align*}
    B_R(y||x) \geq \rho\left( \| y-x \|_x \right) 
    \geq
    \frac{1}{4}\| y-x \|^2_x
    .
\end{align*}
Let $R(q):= - \frac{1}{\gamma} \sum_{h=0}^{H-1}\sum_{s \in S}\sum_{a \in A} \log(q_h(s,a))$.
\\
Recall that adding a linear function to the barrier function does not change the Bregman divergence.
Since $\hat{L}$ is a result of adding a linear function to $R$ we get that 
\begin{align*}
    B_R(q||\widetilde{q})
    =
    B_{-\hat{L}}(q||\widetilde{q})
    =
    -\hat{L}_t(q;c_t) + \hat{L}_t(\widetilde{q};c_t)
    +\nabla \hat{L}_t(\widetilde{q};c_t) \cdot (q- \widetilde{q})
    \leq 
    \epsilon + \nabla \hat{L}_t(\widetilde{q};c_t) \cdot (q- \widetilde{q})
    \leq
    \epsilon
    ,
\end{align*}
where the last step is by the first order optimality condition.
\\
Furthermore, the Hessian is a diagonal matrix $\nabla^2R$, where
${\nabla^2R(q)}_{(s,a,h),(s,a,h)} =  \frac{1}{\gamma q^2_h(s,a)} $.
With us have that 
\begin{align*}
    \|q-\widetilde{q}\|^2_{\nabla^2R(\widetilde{q})}
    =
    \sum_{s \in S} \sum_{a \in A} \sum_{h=0}^{H-1}
    \frac{(q_h(s,a)-\widetilde{q}_h(s,a))^2}{\gamma \widetilde{q}^2_h(s,a)}
    =
    \frac{1}{\gamma}  \sum_{s \in S} \sum_{a \in A} \sum_{h=0}^{H-1}
    \left( \frac{q_h(s,a)}{\widetilde{q}_h(s,a)} -1\right)^2
    .
\end{align*}
we obtain that
\begin{align*}
    \sum_{s \in S} \sum_{a \in A} \sum_{h=0}^{H-1}
    \left( \frac{q_h(s,a)}{\widetilde{q}_h(s,a)} -1\right)^2
    \leq 4 \epsilon \gamma
    . 
    \qquad
    \qedhere
\end{align*}

\end{proof}

We use the above result to derive a bound on the value difference between $\pi_\star$ and $\pi_t$ on the approximated model.
\begin{lemma}[restatement of~\cref{corl:comulative-bound-of-term-2-UD-approx}]\label{lemma:bound-of-term-2-UD-approx}
    For every round $t \in [T]$ and a context $c_t \in \C$, let $\widetilde{q}^t$ be the optimal solution to the maximization problem in~\cref{eq:max-problem-unknown-dynamics}. Suppose that $\hat{q}^t_h \in \mu(\widehat{P}^{c_t}_t)$ satisfies 
    $\hat{L}_t(\hat{q}^t;c_t) - \hat{L}_t(\widetilde{q}^t;c_t) \leq \epsilon
    ,
    $
    and $\epsilon \gamma \le 1/ 16$.
    Then we have that
    \begin{align*}
        V^{\pi^{c_t}_\star}_{\Mhat_t(c_t)}(s_0)
        -
        V^{\pi^{c_t}_t}_{\Mhat_t(c_t)}(s_0)
        % = &
        % \sum_{h=0}^{H-1}\sum_{s \in S}\sum_{a \in A} \hat{q}^{t,\star}_h(s,a) \cdot  \hat{f}_t(c_t,s,a) -\sum_{h=0}^{H-1}\sum_{s \in S}\sum_{a \in A}\hat{q}^t_h(s,a) \cdot \hat{f}_t(c_t,s,a)
        %  \\
         \leq 
         \frac{H|S||A|}{\gamma}- \sum_{h=0}^{H-1}\sum_{s \in S}\sum_{a \in A} \frac{\hat{q}^{t,\star}_h(s,a)}{2\gamma \cdot \hat{q}^t_h(s,a)} 
         +
         2 \sqrt{\epsilon \gamma H}
         ,
    \end{align*}
    where $\hat{q}^t_h(s,a) := q_h(s,a|\pi^{c_t}_t,\widehat{P}^{c_t}_t)$,
    and $\hat{q}^{t,\star}_h(s,a):= q_h(s,a|\pi^{c_t}_\star, \widehat{P}^{c_t}_t)$, $\Mhat_t (c) := (S,A,\widehat{P}^c_t, \hat{f}_t(c,\cdot,\cdot),s_0, H)$ is the estimated CMDP at round $t$, $\pi^{c_t}_t$ is the policy induced by $\widehat{q}^t$, and $\pi_\star = (\pi^c_\star)_{c \in \C}$ is the optimal context-dependent policy of the true CMDP.
\end{lemma}

\begin{proof}
    For every round $t \in [T]$ let $\hat{L}_t(q;c_t)$ denote the objective of the maximization problem in~\cref{eq:max-problem-unknown-dynamics} in round $t$, i.e.,
    \begin{align*}
        \hat{L}_t(q;c_t) = \sum_{h=0}^{H-1}\sum_{s \in S}\sum_{a \in A}q_h(s,a)\cdot \hat{f}_t(c_t,s,a) + \frac{1}{\gamma} \sum_{h=0}^{H-1}\sum_{s \in S}\sum_{a \in A} \log(q_h(s,a)).
    \end{align*}
    Thus the first-order derivation is
    \begin{align*}
         \hat{L}_t^\prime(q;c_t) = \sum_{h=0}^{H-1}\sum_{s \in S}\sum_{a \in A}\left(\hat{f}_t(c_t,s,a) + \frac{1}{\gamma \cdot q_h(s,a)} \right).
    \end{align*}
    
    Let $\pi_\star = (\pi^c_\star)_{c \in \C}$ denote an optimal context-dependent policy for the true CMDP. For every round $t$, the occupancy measures $\hat{q}^{t,\star}_h(s,a):= q_h(s,a|\pi^{c_t}_\star, \widehat{P}^{c_t}_t)$ is a feasible solution (since $\hat{q}^{t,\star} \in \mu(\widehat{P}^{c_t}_t)$). Since $\widetilde{q}^t$ is the optimal solution, the following holds by first order optimality conditions. 
    \begin{align*}
         &\sum_{h=0}^{H-1}\sum_{s \in S}\sum_{a \in A}\left(\hat{f}_t(c_t,s,a) + \frac{1}{\gamma \cdot \widetilde{q}^t_h(s,a)} \right)\left(\hat{q}^{t,\star}_h(s,a) -\widetilde{q}^t_h(s,a)\right) \leq 0
         \\
         \implies &
         \sum_{h=0}^{H-1}\sum_{s \in S}\sum_{a \in A} \hat{q}^{t,\star}_h(s,a) \cdot \left( \hat{f}_t(c_t,s,a) + \frac{1}{\gamma \cdot \widetilde{q}^t_h(s,a)}\right) -\sum_{h=0}^{H-1}\sum_{s \in S}\sum_{a \in A}\widetilde{q}^t_h(s,a) \cdot \hat{f}_t(c_t,s,a) -\frac{H|S||A|}{\gamma} \leq 0
         \\
         \implies & 
         \sum_{h=0}^{H-1}\sum_{s \in S}\sum_{a \in A} \hat{q}^{t,\star}_h(s,a) \cdot  \hat{f}_t(c_t,s,a)
         -\sum_{h=0}^{H-1}\sum_{s \in S}\sum_{a \in A}\widetilde{q}^t_h(s,a) \cdot \hat{f}_t(c_t,s,a) \leq \frac{H|S||A|}{\gamma} - \sum_{h=0}^{H-1}\sum_{s \in S}\sum_{a \in A} \frac{\hat{q}^{t,\star}_h(s,a)}{\gamma \cdot \widetilde{q}^t_h(s,a)} .
    \end{align*}
    By definition we have for every round $t$,
    \begin{align*}
        V^{\pi^{c_t}_\star}_{\Mhat_t(c_t)}(s_0)
        -
        V^{\PiOptReg}_{\Mhat_t(c_t)}(s_0)
        = &
        \sum_{h=0}^{H-1}\sum_{s \in S}\sum_{a \in A} \hat{q}^{t,\star}_h(s,a) \cdot  \hat{f}_t(c_t,s,a) -\sum_{h=0}^{H-1}\sum_{s \in S}\sum_{a \in A}\widetilde{q}^t_h(s,a) \cdot \hat{f}_t(c_t,s,a)
        \\
        \leq &
        \frac{H|S||A|}{\gamma} - \sum_{h=0}^{H-1}\sum_{s \in S}\sum_{a \in A} \frac{\hat{q}^{t,\star}_h(s,a)}{\gamma \cdot \widetilde{q}^t_h(s,a)} 
        \\
        \tag{ By~\cref{lemma:approx-iterates}}
        \leq &
        \frac{H|S||A|}{\gamma} - (1-2\sqrt{\epsilon \gamma})\sum_{h=0}^{H-1}\sum_{s \in S}\sum_{a \in A} \frac{\hat{q}^{t,\star}_h(s,a)}{\gamma \cdot\hat{q}^t_h(s,a)} 
        \\
        \tag{For $ 2\sqrt{\epsilon \gamma} \leq \frac{1}{2}$ }
        \leq &
        \frac{H|S||A|}{\gamma} - \frac{1}{2\gamma}\sum_{h=0}^{H-1}\sum_{s \in S}\sum_{a \in A} \frac{\hat{q}^{t,\star}_h(s,a)}{\hat{q}^t_h(s,a)}
        .
    \end{align*}
    In addition we have that
    \begin{align*}
        V^{\PiOptReg}_{\Mhat_t(c_t)}(s_0)
        -
        V^{\pi^{c_t}_t}_{\Mhat_t(c_t)}(s_0)
        &
        =
        \sum_{h=0}^{H-1}\sum_{s \in S}\sum_{a \in A} 
        (\widetilde{q}^t_h(s,a) - \hat{q}^{t}_h(s,a)) \cdot  \hat{f}_t(c_t,s,a)
        \\
        \tag{By Cauchy–Schwarz inequality}
        &
        \le
        \sqrt{
        \sum_{h=0}^{H-1}\sum_{s \in S}\sum_{a \in A} 
        \frac{(\widetilde{q}^t_h(s,a) - \hat{q}^{t}_h(s,a))^2}{(\widetilde{q}^t_h(s,a))^2}
        }
        \sqrt{
        \sum_{h=0}^{H-1}\sum_{s \in S}\sum_{a \in A}
        \widetilde{q}^t_h(s,a)^2
        \hat{f}_t(c_t,s,a)^2
        }
        \\
        \tag{\cref{lemma:approx-iterates}}
        &
        \le
        2 \sqrt{\epsilon \gamma}
        \sqrt{
        \sum_{h=0}^{H-1}\sum_{s \in S}\sum_{a \in A}
        \widetilde{q}^t_h(s,a)^2
        \hat{f}_t(c_t,s,a)^2
        }
        \\
        \tag{$\widetilde{q}^2 \le \widetilde{q}$}
        &
        \le
        2 \sqrt{\epsilon \gamma}
        \sqrt{
        \sum_{h=0}^{H-1}\sum_{s \in S}\sum_{a \in A}
        \widetilde{q}^t_h(s,a)
        \hat{f}_t(c_t,s,a)^2
        }
        \\
        &
        \le
        2 \sqrt{\epsilon \gamma H}
        .
    \end{align*}
    Combining both bounds concludes the proof.

\end{proof}

We are now ready to derive our regret bound, by combining the above lemma with our previous value bounds. The proof is almost identical to the proof of~\cref{thm:regret-main}.

\begin{theorem*}[restatement of~\cref{thm:regret-approx-main}]\label{thm:regret-approx}
    For any $\delta \in (0,1)$ let 
    $\gamma = \sqrt{\frac{|S||A|T}{62H^3 \left( 2\Regrv_{TH}(\OLSR^\F) + \Regrv_{TH}(\OLLR^{\Fp}) + 18H \log(2H/\delta) \right) }}$. In addition, suppose that at each round $t$ we have an $\epsilon$-approximation to the optimal solution of~\cref{eq:max-problem-unknown-dynamics} for $\epsilon = \frac{1}{16 \gamma T}$. 
    %such that $\epsilon \cdot \gamma \leq 1/16$.
    Then, with probability at least $1-\delta$,
    \begin{align*}
        \Regrv_T(\text{Approximated OMG-CMDP!})\leq 
        \widetilde{O}\left(H^{2.5} \sqrt{ T|S||A| 
        \left( 
        \Regrv_{TH}(\OLSR^\F) + \Regrv_{TH}(\OLLR^{\Fp}) + H \log\delta^{-1} \right)}\right)
        .
    \end{align*}
\end{theorem*}

\begin{proof}
    %For all $t$, let $q^t_h(s,a):= q_h(s,a|\pi^{c_t}_t, P^{c_t}_\star)$.
    By~\cref{lemma:lsro-regret}, with probability at least $1-\delta/2$, it holds that
    \begin{align*}
        \sum_{t=1}^T \mathop{\E}_{\pi^{c_t}_t, P^{c_t}_\star} \Bigg[\sum_{h=0}^{H-1}\left(\hat{f}_t(c_t,s_h,a_h) - f_\star(c_t,s_h,a_h)\right)^2 \Bigg|s_0\Bigg]
        \leq
        2
        \Regrv_{TH}(\OLSR^\F)
        + 16 H \log(2/\delta).
    \end{align*}
    By~\cref{lemma:llrs-regret}, with probability at least $1-\delta/2$, it holds that
    \begin{align*}
        \sum_{t=1}^T 
        \mathop{\E}_{\pi^{c_t}_t, P^{c_t}_\star} \Bigg[
        \sum_{h=0}^{H-1}D^2_H(P^{c_t}_\star(\cdot|s_{h},a_{h}),\widehat{P}^{c_t}_t(\cdot|s_{h},a_{h}))
        \Bigg| s_0\Bigg]
        \leq 
        \Regrv_{TH}(\OLLR^{\Fp}) + 2H \log(2H/\delta).
    \end{align*}
    We prove a regret bound under those two good events.
    \begingroup\allowdisplaybreaks
    \begin{align*}
        &\Regrv_T(\text{Approximated OMG-CMDP!})
        \\
        = &
        \sum_{t=1}^T V^{\pi^{c_t}_\star}_{\M(c_t)}(s_0)
        - V^{\pi^{c_t}_t}_{\M(c_t)}(s_0)
        \\
        = &
        \sum_{t=1}^T 
        V^{\pi^{c_t}_\star}_{\M(c_t)}(s_0)
        -
        V^{\pi^{c_t}_\star}_{\Mhat_t(c_t)}(s_0)
        +
        \sum_{t=1}^T 
        V^{\pi^{c_t}_\star}_{\Mhat_t(c_t)}(s_0)
        -
        V^{\pi^{c_t}_t}_{\Mhat_t(c_t)}(s_0)
        +
        \sum_{t=1}^T 
        V^{\pi^{c_t}_t}_{\Mhat_t(c_t)}(s_0)
        - 
        V^{\pi^{c_t}_t}_{\M(c_t)}(s_0)
        \\
        \tag{By~\cref{corl:bound-of-term-1-UD}, for $\hat{\gamma} = 2\gamma.$}
        \leq &
        \sum_{t=1}^T \sum_{h=0}^{H-1}\sum_{s \in S}\sum_{a \in A}
        \frac{\hat{q}^{t,\star}_h(s,a)}{2\gamma \cdot \hat{q}^t_h(s,a)} 
        \\
        & +
        4
        \gamma\sum_{t=1}^T \mathop{\E}_{\pi^{c_t}_t, P^{c_t}_\star} \Bigg[\sum_{h=0}^{H-1}(\hat{f}_t(c_t,s_h,a_h) - f_\star(c_t,s_h,a_h))^2 \Big|s_0 \Bigg]
        \\
        & +
        58
        \gamma H^4 \sum_{t=1}^T 
         \mathop{\E}_{\pi^{c_t}_t, P^{c_t}_\star} \Bigg[
        \sum_{h=0}^{H-1}
         D^2_H(P^{c_t}_\star(\cdot|s_{h},a_{h}), \widehat{P}^{c_t}_t(\cdot|s_{h},a_{h}))
        \Bigg| s_0\Bigg]
        \\
        \tag{By~\cref{lemma:bound-of-term-2-UD-approx}}
        & +
        \frac{TH|S||A|}{\gamma}- \sum_{t=1}^T\sum_{h=0}^{H-1}\sum_{s \in S}\sum_{a \in A} \frac{\hat{q}^{t,\star}_h(s,a)}{2\gamma \cdot \hat{q}^t_h(s,a)} 
         +
         2 T\sqrt{\epsilon \gamma H}
        \\
        \tag{By~\cref{corl:term-3-UD}}
        & + 
        \frac{TH}{2 p_1}
        \\
        & +
        \frac{p_1}{2}\sum_{t=1}^T \mathop{\E}_{\pi^{c_t}_t, P^{c_t}_\star} \Bigg[ \sum_{h=0}^{H-1}
        \left(\hat{f}_t(c_t,s_h,a_h) - f_\star(c_t,s_h,a_h)\right)^2
        \Bigg|s_0 \Bigg]
        \\
        & +
        \frac{TH}{2 p_2}
        +
        2{p_2}\sum_{t=1}^T 
        \mathop{\E}_{\pi^{c_t}_t, P^{c_t}_\star} \Bigg[
        \sum_{h=0}^{H-1}
         D^2_H(P^{c_t}_\star(\cdot|s_{h},a_{h}), \widehat{P}^{c_t}_t(\cdot|s_{h},a_{h}))
        \Bigg| s_0\Bigg]
        \\
        \tag{By the good events}
        \leq &
        4
        \gamma
        \left( 2\cdot
        \Regrv_{TH}(\OLSR^\F)
        + 16 H \log(2/\delta)\right)
        \\
        & +
        58\gamma H^4 
        \left( \Regrv_{TH}(\OLLR^{\Fp}) + 2H \log(2H/\delta) \right)
        \\
        & +
        \frac{H|S||A|T}{\gamma}          
        +
         2 T\sqrt{\epsilon \gamma H}
        \\
        & + 
        \frac{TH}{2 p_1}
         +
        \frac{p_1}{2}
        \left( 2\cdot
        \Regrv_{TH}(\OLSR^\F)
        + 16 H \log(2/\delta) \right)
        \\
        & +
        \frac{TH}{2 p_2}
        +
        2{p_2}
        \left( \Regrv_{TH}(\OLLR^{\Fp}) + 2H \log(2H/\delta)\right)
        \\
        \leq &
        \gamma \cdot 62 H^4 
        \left( 2\cdot
        \Regrv_{TH}(\OLSR^\F) + \Regrv_{TH}(\OLLR^{\Fp}) + 18H \log(2H/\delta) \right)
        +
        \frac{H|S||A|T}{\gamma}
        \\
        & +
         2 T\sqrt{\epsilon \gamma H}
        \\
        & + 
        \frac{TH}{2 p_1}
         +
        \frac{p_1}{2}
        \left( 2\cdot
        \Regrv_{TH}(\OLSR^\F)
        + 16 H \log(2/\delta) \right)
        \\
        & +
        \frac{TH}{2 p_2}
        +
        2{p_2}
        \left( \Regrv_{TH}(\OLLR^{\Fp}) + 2H \log(2H/\delta)\right)
        \\
        \tag{For $\gamma = \sqrt{\frac{|S||A|T}{62H^3 \left( 2\cdot
        \Regrv_{TH}(\OLSR^\F) + \Regrv_{TH}(\OLLR^{\Fp}) + 18H \log(2H/\delta) \right) }}$}
        = &
        2H^{2.5} \sqrt{ 62 T|S||A| 
        \left( 2\cdot
        \Regrv_{TH}(\OLSR^\F) + \Regrv_{TH}(\OLLR^{\Fp}) + 18H \log(2H/\delta) \right)}
        \\
        & +
         2 T^{1.25} \epsilon^{0.5} \left(\frac{|S||A|}{62H \left( 2\cdot
        \Regrv_{TH}(\OLSR^\F) + \Regrv_{TH}(\OLLR^{\Fp}) + 18H \log(2H/\delta) \right) }\right)^{1/4}
        \\
        \tag{For $p_1 = \sqrt{\frac{TH}{2\cdot
        \Regrv_{TH}(\OLSR^\F)
        + 16 H \log(2/\delta)}}$}
        & + 
        \sqrt{TH
        \left( 2\cdot
        \Regrv_{TH}(\OLSR^\F)
        + 16 H \log(2/\delta) \right)}
        \\
        \tag{For $p_2 = \sqrt{\frac{TH}{4\left( \Regrv_{TH}(\OLLR^{\Fp}) + 2H \log(2H/\delta)\right)}}$}
        & +
        2
        \sqrt{TH\left( \Regrv_{TH}(\OLLR^{\Fp}) + 2H \log(2H/\delta)\right)}
        \\
        \tag{For $\epsilon = \frac{1}{16 \gamma T}$}
        = &
        \widetilde{O}\left(H^{2.5} \sqrt{ T|S||A| 
        \left( 
        \Regrv_{TH}(\OLSR^\F) + \Regrv_{TH}(\OLLR^{\Fp}) + H \log\delta^{-1} \right)}\right)
        .
    \end{align*}
    \endgroup
    Since the good events hold with probability at least $1-\delta$, so is the regret bound above.
\end{proof}

\end{document}